\documentclass{article}

\makeatletter
\newif\if@neuripsfinal\@neuripsfinaltrue
\newif\if@natbib\@natbibtrue
\newif\if@preprint\@preprintfalse
\newif\if@main\@maintrue
\newif\if@position\@positionfalse
\newif\if@eandd\@eanddfalse
\newif\if@creativeai\@creativeaifalse
\newif\if@anonymous\@anonymousfalse
\newif\if@workshop\@workshopfalse
\newcommand{\@trackname}{\@neuripsordinal\ Conference on Neural Information Processing Systems (NeurIPS \@neuripsyear).}

\newcommand{\@neuripsordinal}{40th}
\newcommand{\@neuripsyear}{2026}
\newcommand{\@neuripslocation}{Sydney}

\usepackage{environ}
\newcommand{\acksection}{\section*{Acknowledgments and Disclosure of Funding}}
\NewEnviron{ack}{%
  \acksection
  \BODY
}

\if@natbib
  \RequirePackage{natbib}
\fi

\usepackage[verbose=true,letterpaper]{geometry}
\AtBeginDocument{
  \newgeometry{
    textheight=9in,
    textwidth=5.5in,
    top=1in,
    headheight=12pt,
    headsep=25pt,
    footskip=30pt
  }
  \@ifpackageloaded{fullpage}
    {\PackageWarning{neurips_2026}{fullpage package not allowed! Overwriting formatting.}}
    {}
}

\renewcommand{\normalsize}{%
  \@setfontsize\normalsize\@xpt\@xipt
  \abovedisplayskip      7\p@ \@plus 2\p@ \@minus 5\p@
  \abovedisplayshortskip \z@ \@plus 3\p@
  \belowdisplayskip      \abovedisplayskip
  \belowdisplayshortskip 4\p@ \@plus 3\p@ \@minus 3\p@
}
\normalsize
\renewcommand{\small}{%
  \@setfontsize\small\@ixpt\@xpt
  \abovedisplayskip      6\p@ \@plus 1.5\p@ \@minus 4\p@
  \abovedisplayshortskip \z@  \@plus 2\p@
  \belowdisplayskip      \abovedisplayskip
  \belowdisplayshortskip 3\p@ \@plus 2\p@   \@minus 2\p@
}
\renewcommand{\footnotesize}{\@setfontsize\footnotesize\@ixpt\@xpt}
\renewcommand{\scriptsize}{\@setfontsize\scriptsize\@viipt\@viiipt}
\renewcommand{\tiny}{\@setfontsize\tiny\@vipt\@viipt}
\renewcommand{\large}{\@setfontsize\large\@xiipt{14}}
\renewcommand{\Large}{\@setfontsize\Large\@xivpt{16}}
\renewcommand{\LARGE}{\@setfontsize\LARGE\@xviipt{20}}
\renewcommand{\huge}{\@setfontsize\huge\@xxpt{23}}
\renewcommand{\Huge}{\@setfontsize\Huge\@xxvpt{28}}

\renewcommand{\tiny}{\fontsize{6pt}{7pt}\selectfont}

\renewcommand{\scriptsize}{\fontsize{7pt}{8pt}\selectfont}

\renewcommand{\footnotesize}{\fontsize{8pt}{9.5pt}\selectfont}

\providecommand{\section}{}
\renewcommand{\section}{%
  \@startsection{section}{1}{\z@}%
                {-2.0ex \@plus -0.5ex \@minus -0.2ex}%
                { 1.5ex \@plus  0.3ex \@minus  0.2ex}%
                {\large\bf\raggedright}%
}
\providecommand{\subsection}{}
\renewcommand{\subsection}{%
  \@startsection{subsection}{2}{\z@}%
                {-1.8ex \@plus -0.5ex \@minus -0.2ex}%
                { 0.8ex \@plus  0.2ex}%
                {\normalsize\bf\raggedright}%
}
\providecommand{\subsubsection}{}
\renewcommand{\subsubsection}{%
  \@startsection{subsubsection}{3}{\z@}%
                {-1.5ex \@plus -0.5ex \@minus -0.2ex}%
                { 0.5ex \@plus  0.2ex}%
                {\normalsize\bf\raggedright}%
}
\providecommand{\paragraph}{}
\renewcommand{\paragraph}{%
  \@startsection{paragraph}{4}{\z@}%
                {1.5ex \@plus 0.5ex \@minus 0.2ex}%
                {-1em}%
                {\normalsize\bf}%
}
\providecommand{\subparagraph}{}
\renewcommand{\subparagraph}{%
  \@startsection{subparagraph}{5}{\z@}%
                {1.5ex \@plus 0.5ex \@minus 0.2ex}%
                {-1em}%
                {\normalsize\bf}%
}

\newlength{\@neuripsabovecaptionskip}
\newlength{\@neuripsbelowcaptionskip}

\renewenvironment{table}
  {\setlength{\abovecaptionskip}{\@neuripsbelowcaptionskip}%
   \setlength{\belowcaptionskip}{\@neuripsabovecaptionskip}%
   \@float{table}}
  {\end@float}

\renewcommand{\footnoterule}{\kern-3\p@ \hrule width 12pc \kern 2.6\p@}
\def\@listi  {\leftmargin\leftmargini}
\def\@listii {\leftmargin\leftmarginii
              \labelwidth\leftmarginii
              \advance\labelwidth-\labelsep
              \topsep  2\p@ \@plus 1\p@    \@minus 0.5\p@
              \parsep  1\p@ \@plus 0.5\p@ \@minus 0.5\p@
              \itemsep \parsep}
\def\@listiii{\leftmargin\leftmarginiii
              \labelwidth\leftmarginiii
              \advance\labelwidth-\labelsep
              \topsep    1\p@ \@plus 0.5\p@ \@minus 0.5\p@
              \parsep    \z@
              \partopsep 0.5\p@ \@plus 0\p@ \@minus 0.5\p@
              \itemsep \topsep}
\def\@listiv {\leftmargin\leftmarginiv
              \labelwidth\leftmarginiv
              \advance\labelwidth-\labelsep}
\def\@listv  {\leftmargin\leftmarginv
              \labelwidth\leftmarginv
              \advance\labelwidth-\labelsep}
\def\@listvi {\leftmargin\leftmarginvi
              \labelwidth\leftmarginvi
              \advance\labelwidth-\labelsep}

\providecommand{\maketitle}{}
\renewcommand{\maketitle}{%
  \par
  \begingroup
    \renewcommand{\thefootnote}{\fnsymbol{footnote}}
    \renewcommand{\@makefnmark}{\hbox to \z@{$^{\@thefnmark}$\hss}}
    \long\def\@makefntext##1{%
      \parindent 1em\noindent
      \hbox to 1.8em{\hss $\m@th ^{\@thefnmark}$}##1
    }
    \thispagestyle{empty}
    \@maketitle
    \@thanks
    \@notice
  \endgroup
  \let\maketitle\relax
  \let\thanks\relax
}

\newcommand{\@toptitlebar}{
  \hrule height 4\p@
  \vskip 0.25in
  \vskip -\parskip%
}
\newcommand{\@bottomtitlebar}{
  \vskip 0.29in
  \vskip -\parskip
  \hrule height 1\p@
  \vskip 0.09in%
}

\providecommand{\@maketitle}{}
\renewcommand{\@maketitle}{%
  \vbox{%
    \hsize\textwidth
    \linewidth\hsize
    \vskip 0.1in
    \@toptitlebar
    \centering
    {\LARGE\bf \@title\par}
    \@bottomtitlebar
    \if@anonymous
      \begin{tabular}[t]{c}\bf\rule{\z@}{24\p@}
        Anonymous Author(s) \\
        Affiliation \\
        Address \\
        \texttt{email} \\
      \end{tabular}%
    \else
      \def\And{%
        \end{tabular}\hfil\linebreak[0]\hfil%
        \begin{tabular}[t]{c}\bf\rule{\z@}{24\p@}\ignorespaces%
      }
      \def\AND{%
        \end{tabular}\hfil\linebreak[4]\hfil%
        \begin{tabular}[t]{c}\bf\rule{\z@}{24\p@}\ignorespaces%
      }
      \begin{tabular}[t]{c}\bf\rule{\z@}{24\p@}\@author\end{tabular}%
    \fi
    \vskip 0.3in \@minus 0.1in
  }
}

\newcommand{\ftype@noticebox}{8}
\newcommand{\@notice}{%
  \enlargethispage{2\baselineskip}%
  \@float{noticebox}[b]%
    \footnotesize\@noticestring%
  \end@float%
}

\renewenvironment{abstract}%
{%
  \vskip 0.075in%
  \centerline%
  {\large\bf Abstract}%
  \vspace{0.5ex}%
  \begin{quote}%
}
{
  \par%
  \end{quote}%
  \vskip 1ex%
}

\newcommand{\answerTODO}[1][]{\textcolor{red}{\bf [TODO]}}
\newcommand{\justificationTODO}[1][]{\textcolor{red}{\bf [TODO]}}

\providecommand{\@trackname}{\@neuripsordinal\ Conference on Neural Information Processing Systems (NeurIPS \@neuripsyear).}
\if@preprint
  \newcommand{\@noticestring}{%
    Preprint.%
  }
\else
  \if@neuripsfinal
    \newcommand{\@noticestring}{
      \@trackname
    }
  \else
    \newcommand{\@noticestring}{%
      Submitted to \@neuripsordinal\/ Conference on Neural Information Processing Systems (NeurIPS \@neuripsyear). Do not distribute.%
    }

    \NewEnviron{hide}{}

    \RequirePackage{lineno}
    \linenumbers

    \AtBeginDocument{%
      \@ifpackageloaded{amsmath}{%
        \newcommand*\patchAmsMathEnvironmentForLineno[1]{%
          \expandafter\let\csname old#1\expandafter\endcsname\csname #1\endcsname
          \expandafter\let\csname oldend#1\expandafter\endcsname\csname end#1\endcsname
          \renewenvironment{#1}%
                          {\linenomath\csname old#1\endcsname}%
                          {\csname oldend#1\endcsname\endlinenomath}%
        }%
        \newcommand*\patchBothAmsMathEnvironmentsForLineno[1]{%
          \patchAmsMathEnvironmentForLineno{#1}%
          \patchAmsMathEnvironmentForLineno{#1*}%
        }%
        \patchBothAmsMathEnvironmentsForLineno{equation}%
        \patchBothAmsMathEnvironmentsForLineno{align}%
        \patchBothAmsMathEnvironmentsForLineno{flalign}%
        \patchBothAmsMathEnvironmentsForLineno{alignat}%
        \patchBothAmsMathEnvironmentsForLineno{gather}%
        \patchBothAmsMathEnvironmentsForLineno{multline}%
      }
      {}
    }
  \fi
\fi

\makeatother

\usepackage[utf8]{inputenc}
\usepackage[T1]{fontenc}
\usepackage{textcomp}

\usepackage[hypertexnames=false,colorlinks=true,linkcolor=black,citecolor=black,urlcolor=black]{hyperref}
\usepackage{url}

\usepackage{booktabs}
\usepackage{multirow}
\usepackage{array}
\usepackage{float}
\usepackage{placeins}
\usepackage{amsmath}
\usepackage{amsfonts}
\usepackage{amssymb}
\usepackage{mathtools}
\usepackage{nicefrac}

\usepackage{amsthm}
\newtheorem{definition}{Definition}
\newtheorem{lemma}{Lemma}
\newtheorem{proposition}{Proposition}
\newtheorem{assumption}{Assumption}

\usepackage{cleveref}
\crefname{assumption}{Assumption}{Assumptions}

\usepackage{enumitem}

\usepackage{graphicx}

\usepackage{microtype}
\usepackage{xcolor}
\usepackage{eso-pic}

\makeatletter
\renewcommand{\@notice}{%
  \AddToShipoutPictureFG*{%
    \AtPageLowerLeft{%
      \raisebox{0.34in}{\hspace*{0.5in}\parbox{\textwidth}{\footnotesize \@noticestring}}%
    }%
  }%
}
\makeatother

\definecolor{dposcol}{rgb}{0.00,0.55,0.00}
\definecolor{dnegcol}{rgb}{0.75,0.00,0.00}
\newcommand{\dpos}[1]{{\scriptsize\textcolor{dposcol}{$#1$}}}
\newcommand{\dneg}[1]{{\scriptsize\textcolor{dnegcol}{$#1$}}}
\newcommand{\dposb}[1]{{\scriptsize\textcolor{dposcol}{$\boldsymbol{#1}$}}}

\newcommand{\sysname}{\textsc{Evidence-RL}}
\newcommand{\ced}{\textsc{CED}}

\title{Evidence-RL: Towards Evidence-intensive Visual Reasoning}

\author{%
  \bfseries\small Haojie Huang\textsuperscript{1,2,*}\quad
  Xinlei Yu\textsuperscript{1,*}\quad
  Chengming Xu\textsuperscript{3,*}\quad
  Zhangquan Chen\textsuperscript{4}\quad
  Cheng Yang\textsuperscript{5} \\
  \bfseries\small Qingdong He\textsuperscript{5}\quad
  Yu Yang\textsuperscript{2}\quad
  Jiangning Zhang\textsuperscript{2}\quad
  Xiaobin Hu\textsuperscript{1,\textdagger} \\[2pt]
  \normalfont\small
  \texttt{greatesthhj@zju.edu.cn}\quad
  \texttt{ben0xiaobin0hu1@nus.edu.sg} \\[2pt]
  \normalfont\small
  \textsuperscript{1}National University of Singapore\quad
  \textsuperscript{2}Zhejiang University \\
  \textsuperscript{3}Fudan University\quad
  \textsuperscript{4}Tsinghua University\quad
  \textsuperscript{5}Tencent \\[2pt]
  \normalfont\footnotesize
  \textsuperscript{*}Equal contribution.\quad
  \textsuperscript{\textdagger}Corresponding author.
}

\begin{document}

\maketitle

\begin{abstract}
Vision-Language Models (VLMs) should answer from concrete image evidence rather than language priors, dataset shortcuts, or irrelevant visual context. Existing perception-aware post-training methods encourage image use through global perturbations or attention proxies, but they do not test whether a sampled answer causally depends on the local evidence that supports it. We propose \textbf{Counterfactual Evidence Disentanglement} (\ced{}), a training-time evidence audit for VLM grounding. For each response, \ced{} neutralizes an object-centric Evidence Region and compares the resulting support drop against matched non-evidence Regions. We combine this signal with answer correctness inside GRPO, rewarding correct answers that rely on the evidence path rather than shortcut or nuisance paths. \ced{} uses weak object-level proposals, requires no question-specific evidence annotations, and adds no inference-time overhead. Across nine public benchmarks and four backbones, \ced{} outperforms prior RL-based post-training methods, with targeted analyses verifying its object-centric signal. Project page including code and checkpoint is available at \href{https://evidencerl.github.io}{\texttt{https://evidencerl.github.io}}.
\end{abstract}

\section{Introduction}
\label{sec:introduction}

Vision-Language Models (VLMs) are increasingly used as general-purpose visual reasoners~\citep{bai2025qwen25vltechnicalreport,bai2025qwen3,wang2025internvl3,li2024llava}, but correct-looking answers need not be visually grounded. A model may answer a counting or spatial question by inspecting the relevant evidence, or by relying on language priors, dataset regularities, and scene-level common sense, a failure pattern documented in hallucination, counting, spatial, and shortcut-reasoning evaluations~\citep{li2023evaluating,paiss2023teaching,wang2024picture,guan2024hallusionbench,rahmanzadehgervi2024vision,yin2026freak}. Post-training can amplify either behavior depending on what the training signal can observe.

Recent perception-aware methods attempt to reintroduce the image into the post-training loop. PAPO~\citep{wang2025papo} compares model behavior under the original image and a corrupted image, while VPPO~\citep{huang2025vppo} uses attention-derived visual dependency to adjust the training signal. These methods are important because they move beyond text-only supervision. However, they still target \emph{coarse visual dependence} as a weaker visual anchor rather than \emph{causal evidence dependence}.  In general, the image is not a single cause. It contains the target evidence that supports the answer, irrelevant context, visually salient distractors, and artifacts introduced by masking, zeroing, or noise. A global image perturbation can show that the model is sensitive to the image, but it cannot tell whether the current answer specifically depends on the local evidence that would falsify it if removed. Attention-based dependency is also an internal routing proxy, not a direct test of whether the answer is supported by the relevant evidence.

\begin{figure}[t]
    \centering
    \includegraphics[width=\linewidth]{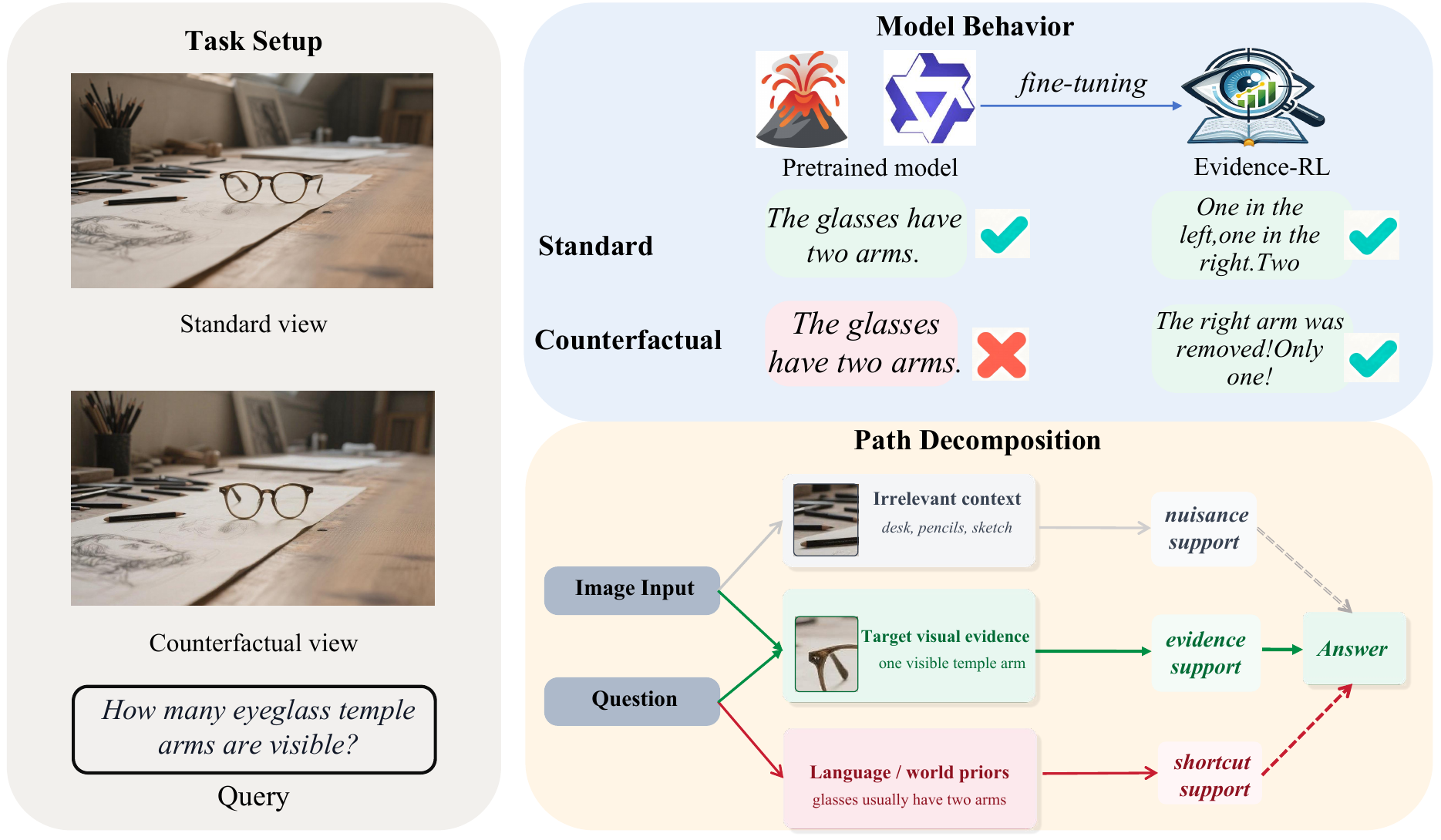}
    \vspace{-4mm}
    \caption{\textbf{Left:} Task setup with a counting query over two views of the same scene. \textbf{Upper right:} A pretrained VLM gives the same prior-based answer for both views; \sysname{} inspects the visual evidence and adjusts its answer when the scene changes. \textbf{Lower right:} Causal path decomposition: a candidate answer can be supported by the target visual evidence, by irrelevant context, or by language priors, which motivates \ced{} to test whether the response depends specifically on the evidence path.}
    \label{fig:causal-graph}
\end{figure}

The path decomposition in \Cref{fig:causal-graph} causally formalizes this interference, following the structural-causal view that interventions should isolate the path being tested~\citep{pearl2009causality}. A candidate answer can be produced by language and world priors, by irrelevant visual context or distractors, or by the target evidence that actually supports the answer. The shortcut problem arises because these paths meet at the same observed output: the same correct answer may be caused by the evidence path or by a non-evidence path. This suggests a more general requirement for grounded VLM post-training. Instead of asking only whether the model uses the image, we should ask whether the answer depends on the particular visual evidence that supports it, while discounting nuisance sensitivity to other regions and to the intervention itself. Without an image-conditional probe, a training signal computed on the model's text output alone cannot distinguish answers that traverse the target-evidence path from those that traverse a shortcut path; we make this consequence precise as \emph{evidence-closed self-evolution} in \Cref{sec:self-evolve} where grounding becomes formally unidentifiable.

Based on this causal view, we propose \textbf{Counterfactual Evidence Disentanglement} (\ced{}), a novel post-training framework that trains VLMs to rely on evidence-specific counterfactual dependence rather than coarse image sensitivity. The core idea is to turn the causal graph into a training-time counterfactual test: for each sampled response, \ced{} asks whether the model's support for that response decreases more when the object-centric Evidence Region is neutralized than when matched non-evidence Regions are neutralized. A response is therefore favored only when its correctness is tied to the evidence path, rather than preserved through language priors, irrelevant context, or generic sensitivity to visual perturbations. We implement this test by intervening on visual tokens with mean replacement, which removes region-specific information while preserving the local feature baseline, making the intervention a structured local removal of evidence rather than zeroing, random noise, or global corruption. The resulting evidence signal is combined with answer correctness inside GRPO~\citep{shao2024deepseekmath}, encouraging correct answers that are causally grounded in the image without requiring question-specific region annotations or adding inference-time overhead.

We evaluate \ced{} on public VLM benchmarks spanning VQA, hallucination detection, counting, spatial reasoning, and perception-heavy visual reasoning, in which \ced{} outperforms prior RL-based post-training methods, with clear gains on evaluations that require resisting language priors and grounding answers in visual evidence. Targeted analyses confirm that \ced{} captures evidence-specific dependence, that mean replacement provides a cleaner counterfactual intervention than common alternatives, and that the signal follows plausible target evidence rather than arbitrary regions. Throughout the paper we use \ced{} for the counterfactual evidence diagnostic itself (the reward signal) and \sysname{} for the full post-training pipeline that combines \ced{} with GRPO (\Cref{sec:method}). Our contributions can be summarized as follows:

\begin{itemize}
    \item We articulate the gap between \emph{coarse visual dependence} and \emph{counterfactual evidence dependence}, and formalize visual grounding as a structural causal graph in which the same observed answer can arise from the target-evidence path or from shortcut and nuisance paths. We further characterize why text-only self-evolution can remain evidence-closed.

    \item We propose \ced{}, a counterfactual evidence audit that tests whether a candidate answer depends on a proposed Evidence Region relative to matched non-evidence Regions. The audit is defined on the causal graph rather than tied to a specific evidence type; we instantiate it for spatial-region evidence and outline how the same recipe extends to attribute-level and relational evidence by varying the intervention.

    \item We integrate \ced{} into GRPO and validate it through signal diagnostics, controlled ablations (correctness-only, matched baselines, intervention type, proposal robustness), and nine public benchmarks across four backbones. The results show that targeting counterfactual evidence dependence transfers across backbones, recovers gains on perception-heavy benchmarks, and adds no inference-time overhead.
\end{itemize}

\vspace{-9pt}
\section{Related Work}
\label{sec:related}

\paragraph{VLM reasoning.}
Vision-language models built on instruction-tuned LLM backbones~\citep{li2024llava,bai2025qwen25vltechnicalreport,bai2025qwen3,wang2025internvl3} have evolved from task-specific perception systems into general-purpose multimodal reasoners, supporting visual QA, counting, spatial reasoning, diagrammatic reasoning, and scientific problem solving~\citep{paiss2023teaching,wang2024picture,lu2022learn,lu2023mathvista,yue2024mmmu}. This progress exposes a persistent gap between answer correctness and visual grounding: a model can produce a plausible answer without relying on the evidence that should justify it. Benchmarks such as HallusionBench~\citep{guan2024hallusionbench}, VLMsAreBlind~\citep{rahmanzadehgervi2024vision}, and FREAK~\citep{yin2026freak} make this failure measurable through hallucination, shortcut reasoning, and fine-grained visual errors. We address this from the post-training side, rewarding answers whose likelihood depends on relevant visual evidence.

\paragraph{Post-training VLMs.}
Reinforcement learning has become a central post-training mechanism for improving reasoning, from PPO-style optimization~\citep{schulman2017proximal,shao2024deepseekmath} to verifiable-reward and self-evolution recipes~\citep{guo2025deepseek,acikgoz2026tool}. Recent VLM methods, including VLM-R1~\citep{shen2025vlm}, ViCrit~\citep{wang2025vicrit}, Perception-R1~\citep{xiao2025perceptionr1}, SophiaVL-R1~\citep{fan2025sophiavl}, and VAPO~\citep{yue2025vapo}, bring this paradigm to multimodal reasoning. The key difficulty is that visual tasks lack cheap per-sample verifiers analogous to code execution or arithmetic checking. Text-only judges can score the plausibility of an answer or rationale but cannot tell whether it was caused by the image, so grounding-oriented RL needs rewards that observe visual evidence.

\paragraph{Grounded Rewards.}
Existing attempts move in this direction but remain limited. Image-conditional rewards compare rollouts under original and perturbed images~\citep{wang2025papo, huang2025vppo, du2026linking} or distill from crop-conditioned teachers~\citep{yuan2026vision}, yet neither tests whether the decisive region supports the answer. Process-level and self-evolving rewards~\citep{fan2025sophiavl, zhang2025viper, acikgoz2026tool, zhao2026robustness} audit the reasoning trace, but remain text-closed and can reward language-prior shortcuts with consistent explanations; we formalize this as \emph{evidence-closed self-evolution} in \Cref{sec:self-evolve}. Counterfactual visual interventions expose grounding failures at inference time~\citep{leng2024vcd,liu2025reducing,li2025treble,wu2026revis}, but usually do not update the model. Our intervention is related to perturbation-based visual explanations~\citep{zeiler2014visualizing,fong2017interpretable,petsiuk2018rise}, but we use the score as a training reward. Closest training-time methods require region annotations in reasoning traces~\citep{sun2026regionreasoner}, process/checklist supervision~\citep{li2026palmr,zhang2025perceptual,wang2025vgr,qin2025chain}, or adversarial text/preference pairs~\citep{chen2025perturbollava,yang2026hii}. In contrast, \ced{} brings region-level counterfactual intervention into RL as a reward, using off-the-shelf object proposals instead of question-specific annotations.

\section{Method}
\label{sec:method}

\Cref{fig:method-overview} summarizes \sysname{} as a training-time pipeline: sample a candidate answer, contrast the proposed Evidence Region with matched non-evidence Regions, intervene in feature space, compute an answer-conditioned evidence margin, and use the resulting reward in GRPO. We now define the counterfactual probe and its integration into post-training.

\begin{figure}[t]
    \centering
    \includegraphics[width=0.95\linewidth]{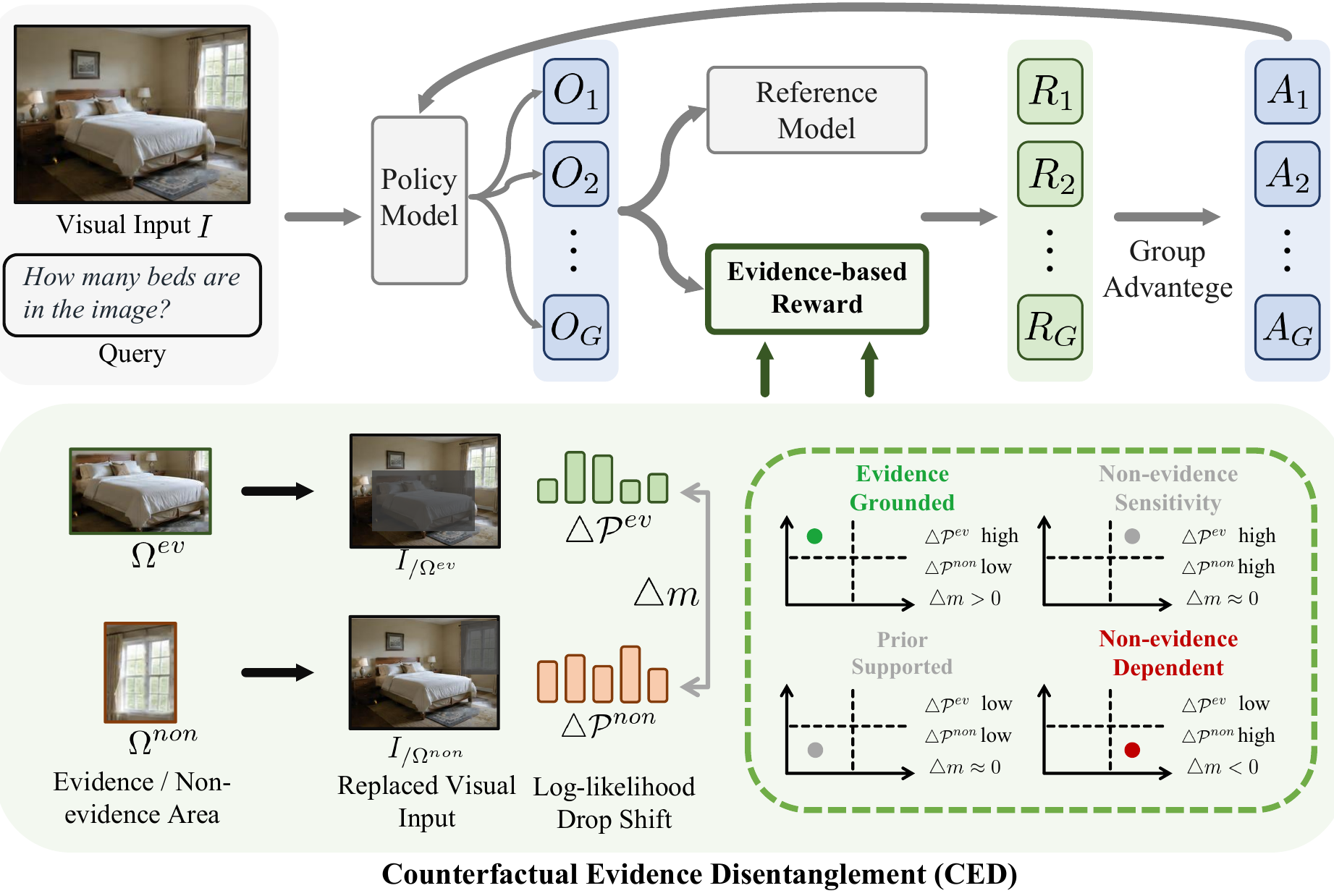}
    \vspace{-2mm}
    \caption{
    Overview of \sysname{} during post-training.
    A policy samples candidate answers, and \ced{} scores each answer by applying the same feature-space intervention to the proposed Evidence Region and to matched non-evidence Regions. The contrastive support drop yields an evidence margin, which is combined with correctness to form the GRPO reward. The counterfactual audit is used only during training; inference uses the trained VLM normally.}
    \label{fig:method-overview}
\end{figure}

\subsection{\ced{}: Counterfactual Evidence Margin}
\label{sec:intervention}

Correctness alone leaves open whether an answer is supported by the image or by language priors. We therefore use ``grounding'' in a narrow sense: an answer is grounded when its support depends on a specific locus of visual evidence. \ced{} turns this evidence sensitivity into a training signal complementary to answer correctness.

Given image $I$, question $q$, and candidate answer $y$, we identify an Evidence Region $\Omega^{\mathrm{ev}}$ hypothesized to contain answer-relevant evidence and $K$ non-evidence Regions $\{\Omega^{\mathrm{non}}_k\}_{k=1}^{K}$. For a region $\Omega$, we intervene after spatial merging by replacing its visual tokens with the mean $\boldsymbol{\mu}_T$ of neighboring tokens:
\begin{equation}
    \tilde{\mathbf{h}}_i =
    \begin{cases}
        \boldsymbol{\mu}_T & \text{if } i \in \mathcal{T}(\Omega), \\
        \mathbf{h}_i & \text{otherwise},
    \end{cases}
\end{equation}
where $\mathcal{T}(\Omega)$ maps a spatial region to visual-token indices. Feature-space mean replacement removes region-specific information while keeping the local representation manifold largely intact, reducing artifacts compared with zeroing or noise injection. We evaluate these alternatives in \Cref{sec:perturbation-ablation} and give a mechanistic account in \Cref{app:why-mean}.

The evidence sensitivity of $y$ to region $\Omega$ is defined as the counterfactual log-likelihood drop:
\begin{equation}
\label{eq:sensitivity}
    s(\Omega) \;=\; \log \pi_\theta(y \mid I, q) \;-\; \log \pi_\theta(y \mid \tilde{I}_{\setminus\Omega}, q).
\end{equation}
If the model can support $y$ without inspecting $\Omega$, then $s(\Omega)$ remains close to zero. To obtain a bounded and contrastive score, we compare the evidence-region sensitivity with reference-region sensitivities:
\begin{equation}
\label{eq:margin}
    m(I,q,y) \;=\; \tanh \!\left( \frac{s(\Omega^{\mathrm{ev}}) - \mu(s^{\mathrm{non}})}{\sigma(s^{\mathrm{non}}) + \epsilon} \right),
\end{equation}
where $s^{\mathrm{non}}=\{s(\Omega^{\mathrm{non}}_k)\}_{k=1}^{K}$, and $\mu(s^{\mathrm{non}})$ and $\sigma(s^{\mathrm{non}})$ are the mean and standard deviation over non-evidence Regions. A positive margin indicates that the candidate answer depends more on the proposed Evidence Region than on area-matched non-evidence Regions.

\paragraph{Non-evidence Regions as a local null.}
Non-evidence Regions are not treated as annotated negatives. They define a sample-local null distribution for how much the candidate answer changes under the same feature-space intervention applied to regions that are area-matched and spatially separated from the proposed Evidence Region. Thus, if a response is sensitive mainly to generic masking artifacts or irrelevant visual context, the Evidence and non-evidence interventions should induce comparable support drops. A high margin is obtained only when the Evidence Region produces a larger drop than this non-evidence distribution. In implementation, the $K$ non-evidence Regions jointly estimate this baseline rather than forming one-to-one evidence--non-evidence pairs.

\paragraph{Response-conditioned local interventions.}
\ced{} is designed to isolate the target-evidence path in \Cref{fig:causal-graph}. It scores the sampled answer $y$ itself, measuring the likelihood drop of that response rather than a distribution-level shift. It also intervenes locally, targeting the region that should carry the relevant evidence instead of perturbing the whole image. Finally, the intervention is applied in feature space after spatial merging, which preserves the representation outside the intervened region and reduces pixel-level encoder artifacts. These choices make the score answer-specific, evidence-localized, and less confounded by intervention noise.
\paragraph{Weakly-supervised evidence proposals.}
The Evidence Region $\Omega^{\mathrm{ev}}$ is a weak spatial proposal resolved from dataset metadata, using object-level COCO annotations~\citep{lin2014microsoft} in our experiments. Such boxes provide spatial priors: the same object box may or may not contain evidence depending on the question. Reference Regions are sampled to match the intervention scale while avoiding the proposed Evidence Region. \ced{} resolves this proposal ambiguity by comparing the counterfactual sensitivity of $\Omega^{\mathrm{ev}}$ against matched Reference Regions, so the evidence signal is determined by model behavior which are expected to transfer to various downstream tasks. Proposal-layer reliability is evaluated in \Cref{sec:boundary}.

\subsection{\sysname{}: CED-Guided GRPO}
\label{sec:ced-in-grpo}

Having defined the evidence margin $m$, we integrate it into GRPO as a correctness-anchored reward. The resulting pipeline, including \ced{} reward shaping and binary-action routing, is denoted by \sysname{}.

\paragraph{Correctness-anchored reward.}
The training reward is:
\begin{equation}
\label{eq:reward}
    R_{\text{train}} = R_{\text{ans}}\cdot g(m) + \varepsilon_{\text{tie}}\,m,
    \qquad
    g(m)=\tfrac{1}{2}\!\left(1 + \tanh\!\left(\tfrac{m}{\tau_g}\right)\right),
\end{equation}
where $\tau_g{=}0.20$, $\varepsilon_{\text{tie}}{=}0.10$, and $R_{\text{ans}}$ is the answer-correctness score defined in Appendix~\ref{app:reproducibility}. When the answer is correct, $g(m)$ favors responses with stronger Evidence Region support. When the answer is wrong and $R_{\text{ans}}\to 0$, the evidence term becomes a bounded tiebreaker, with $|\varepsilon_{\text{tie}}m|\leq\varepsilon_{\text{tie}}$. Thus, after within-group normalization $A_i=(R_{\text{train},i}-\mu)/\sigma$, correctness remains the dominant gradient driver, while \ced{} ranks equally correct rollouts by evidence support.

\section{Experiments}
\label{sec:experiments}

The experiments are designed to validate four properties of our method:(1)the counterfactual reward provides a usable evidence-dependence signal;(2)RL with this signal improves benchmark performance against matched-backbone baselines;(3)the gains transfer across backbones,;(4)the signal remains stable under intervention and proposal perturbations. Unless explicitly stated, ours denotes the Answer-CED variant. The main head-to-head comparison uses Qwen2.5-VL-7B~\citep{bai2025qwen25vltechnicalreport} as the base model for ours; cross-backbone validation further covers Qwen2.5-VL-3B/7B~\citep{bai2025qwen25vltechnicalreport}, Qwen3-VL-8B~\citep{bai2025qwen3}, and Qwen3.5-9B~\citep{qwen35blog}. For reference, the backbone rows also include LLaVA-v1.6-7B~\citep{li2024llava} and InternVL3.5-8B~\citep{wang2025internvl3}. Matched baselines include recent RL-based VLM post-training methods~\citep{yue2025vapo,huang2025vppo,xiao2025perceptionr1,wang2025papo,shen2025vlm,fan2025sophiavl}. Signal-validation diagnostics use Qwen3-VL-8B-Instruct as a frozen probe. Evaluation covers nine public benchmarks spanning grounding, faithfulness, and general reasoning: CountBench~\citep{paiss2023teaching}, SpatialEval~\citep{wang2024picture}, HallusionBench~\citep{guan2024hallusionbench}, VLMsAreBlind~\citep{rahmanzadehgervi2024vision}, FREAK~\citep{yin2026freak}, MathVista~\citep{lu2023mathvista}, MMBench~\citep{liu2024mmbench}, MMMU~\citep{yue2024mmmu}, and ScienceQA~\citep{lu2022learn}, with no overlap between training images and evaluation benchmarks.

\subsection{Validating the Evidence-Dependence Signal}
\label{sec:signal-validity}

We first check whether the reward separates evidence dependence from correctness-only. On counting tasks, 99.5\% of groups have non-constant rewards and 90.0\% contain same-answer trajectories with different rewards, showing that the signal preserves within-group discrimination beyond answer-string matching. Presence tasks show higher zero variance, consistent with their two-action yes/no structure discussed in \Cref{sec:ced-in-grpo}. Cross-model validation and blindfold tests are reported in Appendix~\ref{app:signal-extended}.

\begin{table}[t]
\centering
\small
\setlength{\tabcolsep}{4pt}
\renewcommand{\arraystretch}{1.02}
\setlength{\abovecaptionskip}{4pt}
\setlength{\belowcaptionskip}{6pt}
\caption{Signal-validation statistics on counting and presence tasks.}
\vspace{-2mm}
\label{tab:signal-validation}
\begin{tabular*}{\linewidth}{@{\extracolsep{\fill}}l c c c c c@{}}
\toprule
Task & \shortstack{Non-constant\\raw reward} & \shortstack{Zero\\variance} & \shortstack{Mean reward\\std} & \shortstack{Same answer,\\different reward} & \shortstack{Mean\\negation} \\
\midrule
Counting & \textbf{99.5\%} & 24.0\% & 0.1299 & \textbf{90.0\%} & 0.003 \\
Presence & 21.1\% & 79.0\% & 0.0497 & 5.0\% & 0.777 \\
\bottomrule
\end{tabular*}
\vspace{-3mm}
\end{table}

\subsection{Main Results on Nine Benchmarks}
\label{sec:main-results}
\label{sec:benchmark-transfer}
\label{sec:extended-eval}

\Cref{tab:extended-benchmark} compares ours with VLM backbones and RL-based baselines on nine benchmarks. Benchmarks are grouped into \emph{Grounding} and \emph{General Reasoning}. Each $\Delta$ reports the absolute gain over the corresponding base model, which is Qwen2.5-VL-7B unless the row label specifies another base.

\begin{table}[t]
\centering
\scriptsize
\setlength{\tabcolsep}{3.4pt}
\renewcommand{\arraystretch}{0.92}
\setlength{\abovecaptionskip}{4pt}
\setlength{\belowcaptionskip}{6pt}
\caption{Nine-benchmark evaluation using Qwen2.5-VL-7B as the base model for Ours. Backbone, baseline, and benchmark sources are cited in the setup paragraph. Best \textbf{bold}; second-best \underline{underlined}. Each $\Delta$ row reports absolute gain over the corresponding base.}
\vspace{-2mm}
\label{tab:extended-benchmark}
\resizebox{\linewidth}{!}{%
\begin{tabular}{@{}l ccccc cccc c@{}}
\toprule
& \multicolumn{5}{c}{\textbf{Grounding}}
& \multicolumn{4}{c}{\textbf{General Reasoning}} & \\
\cmidrule(lr){2-6} \cmidrule(lr){7-10}
\textbf{Model}
& CountBench & SpatialEval & Hallusion & \shortstack{VLMs\\AreBlind} & FREAK
& MathVista & MMBench & MMMU & ScienceQA & \textbf{Avg} \\
\midrule
\multicolumn{11}{l}{\emph{VLM backbones}} \\
LLaVA-v1.6-7B
  & 55.60 & 23.60 & 51.00 & 28.23 & 9.50
  & 22.80 & 71.30 & 29.00 & 57.20 & 38.69 \\
InternVL3.5-8B
  & 86.90 & 24.00 & 69.60 & \textbf{54.92} & 12.80
  & 51.90 & 83.40 & 50.80 & \textbf{89.90} & 58.25 \\
Qwen2.5-VL-3B
  & 70.71 & 53.94 & 65.37 & 42.46 & 11.14
  & 52.10 & 80.40 & 50.29 & 74.93 & 55.70 \\
Qwen2.5-VL-7B
  & 81.82 & 59.21 & 68.64 & 46.44 & 12.18
  & 63.10 & 84.73 & 50.63 & 83.98 & 61.19 \\
\midrule
\multicolumn{11}{l}{\emph{RL-based methods}} \\
VAPO-Thinker-7B
  & 86.90 & 60.80 & \textbf{71.20} & 48.62 & 13.10
  & 51.10 & 81.80 & 44.60 & 82.30 & 60.05 \\
\quad $\Delta$
  & \dpos{+5.10} & \dpos{+1.59} & \dpos{+2.60} & \dpos{+2.18} & \dpos{+0.92}
  & \dneg{-12.00} & \dneg{-2.90} & \dneg{-6.00} & \dneg{-1.70}
  & \dneg{-1.14} \\
VPPO-7B
  & 85.90 & 61.80 & 68.20 & 49.35 & 12.70
  & \underline{67.90} & 84.90 & 52.10 & \underline{88.40} & 63.47 \\
\quad $\Delta$
  & \dpos{+4.10} & \dpos{+2.59} & \dneg{-0.40} & \dpos{+2.91} & \dpos{+0.52}
  & \dpos{+4.80} & \dpos{+0.20} & \dpos{+1.50} & \dpos{+4.40}
  & \dpos{+2.28} \\
Perception-R1-7B
  & 84.90 & 59.70 & 66.70 & 45.22 & 12.20
  & 67.10 & 81.70 & 48.10 & 82.30 & 60.88 \\
\quad $\Delta$
  & \dpos{+3.10} & \dpos{+0.49} & \dneg{-1.90} & \dneg{-1.22} & \dpos{+0.02}
  & \dpos{+4.00} & \dneg{-3.00} & \dneg{-2.50} & \dneg{-1.70}
  & \dneg{-0.31} \\
PAPO-G-H
  & \textbf{90.90} & \textbf{64.80} & 69.50 & 47.92 & 12.80
  & \textbf{69.40} & 82.80 & 50.50 & 85.60 & 63.80 \\
\quad $\Delta$
  & \dpos{+9.10} & \dpos{+5.59} & \dpos{+0.90} & \dpos{+1.48} & \dpos{+0.62}
  & \dpos{+6.30} & \dneg{-1.90} & \dneg{-0.10} & \dpos{+1.60}
  & \dpos{+2.61} \\
VLM-R1
  & 74.80 & 56.40 & 66.60 & 44.15 & 11.00
  & 60.30 & 79.90 & 48.60 & 73.60 & 57.26 \\
\quad $\Delta$ vs.\ Qwen2.5-VL-3B
  & \dpos{+4.10} & \dpos{+2.50} & \dpos{+1.20} & \dpos{+1.69} & \dneg{-0.14}
  & \dpos{+8.20} & \dneg{-0.50} & \dneg{-1.70} & \dneg{-1.30}
  & \dpos{+1.56} \\
SophiaVL-R1
  & 82.83 & 61.38 & 66.87 & 47.09 & \underline{24.18}
  & 66.30 & \underline{87.41} & \underline{52.18} & 87.64 & \underline{63.99} \\
\quad $\Delta$
  & \dpos{+1.01} & \dpos{+2.17} & \dneg{-1.77} & \dpos{+0.65} & \dpos{+12.00}
  & \dpos{+3.20} & \dpos{+2.68} & \dpos{+1.55} & \dpos{+3.66}
  & \dpos{+2.80} \\
\midrule
\textbf{Ours}
  & \underline{88.89} & \underline{63.34} & \underline{70.08} & \underline{54.02} & \textbf{26.40}
  & \textbf{69.40} & \textbf{87.53} & \textbf{57.47} & 87.08 & \textbf{67.13} \\
\quad $\Delta$
  & \dpos{+7.07} & \dpos{+4.13} & \dpos{+1.44} & \dpos{+7.58} & \dpos{+14.22}
  & \dpos{+6.30} & \dpos{+2.80} & \dpos{+6.84} & \dpos{+3.10}
  & \dposb{+5.94} \\
\bottomrule
\end{tabular}%
}
\end{table}

Ours obtains the highest average score and the largest mean improvement in the matched-backbone comparison. It is also the only RL method with non-negative $\Delta$ on all nine benchmarks. The largest gains appear on perception-heavy benchmarks, including VLMsAreBlind, FREAK, and MMMU, while MathVista, MMBench, and ScienceQA remain positive. This pattern supports the intended role of the reward: increasing reliance on visual evidence without sacrificing general transfer.

\subsection{Ablations and Diagnostics}
\label{sec:ced-isolation}

Having established the aggregate gains, we now diagnose their source. We first isolate the CED contribution under strictly controlled conditions (\Cref{sec:isolating-ced}), then examine cross-backbone generalization and the Answer-vs-CoT design choice (\Cref{sec:generalization-variant}), and finally stress-test the reward signal against proposal and intervention perturbations (\Cref{sec:signal-robustness}).

\subsubsection{Isolating the CED Contribution}
\label{sec:isolating-ced}

\paragraph{Is the CED reward necessary?}
\Cref{tab:ced-isolation} compares correctness-only, retrained VPPO, retrained PAPO, and Answer-CED under matched conditions: the same Qwen3.5-9B backbone, 15,314-sample training set, router, 2,000-step schedule, and LoRA configuration. Without CED, the same data and compute do not produce transferable improvements. Correctness-only and VPPO produce gains on several grounding benchmarks but substantial negative transfer to general reasoning, yielding net-negative average deltas. PAPO becomes unstable under its global-perturbation KL objective, with 25\% repeated-segment rate and 58\% budget exhaustion at 24k tokens. Answer-CED improves all nine benchmarks over correctness-only ($+12.58$ average) and exceeds matched VPPO by $+13.12$. CED uniquely converts the same data and compute into transferable improvements by restoring evidence-specific within-group discrimination among equally correct rollouts.

\Cref{fig:worked-example} illustrates this mechanism. Two rollouts from the same GRPO group both answer ``4'' correctly: one relies on the prior \emph{sedan $= 2{+}2$} without inspecting the image, while the other detects the unusual wheel configuration ($1$ front $+ 2$ rear $+ 1$ spare) through visual evidence. Correctness-only reward produces no ranking signal while the \ced{} gate $g(m)$ multiplies $R_{\mathrm{ans}}$ by the evidence margin, creating a $7\times$ reward gap that GRPO uses to select the grounded trajectory.

\begin{table}[t]
\centering
\scriptsize
\setlength{\tabcolsep}{3.0pt}
\renewcommand{\arraystretch}{0.96}
\caption{Controlled comparison on Qwen3.5-9B with matched data, router, and compute. $\Delta$ is the mean gain over the frozen base; $\Delta_{\text{CED}}$ is Answer-CED minus correctness-only.}
\label{tab:ced-isolation}
\resizebox{\linewidth}{!}{%
\begin{tabular}{@{}l ccccc cccc cc@{}}
\toprule
& \multicolumn{5}{c}{\textbf{Grounding}}
& \multicolumn{4}{c}{\textbf{General Reasoning}} & & \\
\cmidrule(lr){2-6} \cmidrule(lr){7-10}
\textbf{Method}
& CountBench & SpatialEval & Hallusion & \shortstack{VLMs\\AreBlind} & FREAK
& MathVista & MMBench & MMMU & ScienceQA & \textbf{Avg} & \textbf{$\Delta$} \\
\midrule
$\lambda{=}0$ (corr.-only)
& 89.90 & 40.45 & 57.93 & 72.90 & 18.73
& 73.40 & 79.44 & 36.00 & 82.24 & 61.22 & \dneg{-1.24} \\
VPPO (retrained)
& 86.87 & 38.32 & 58.37 & 72.97 & 19.07
& 72.50 & 79.62 & 37.33 & 81.09 & 60.68 & \dneg{-1.78} \\
PAPO (retrained)$^\dagger$
& 86.87 & 8.59 & 58.46 & 69.96 & 19.40
& 14.00 & 44.87 & 33.22 & 34.24 & 41.07 & \dneg{-21.39} \\
\midrule
\textbf{Answer-CED}
& \textbf{93.90} & \textbf{54.35} & \textbf{78.10} & \textbf{73.10} & \textbf{23.12}
& \textbf{83.75} & \textbf{88.58} & \textbf{75.67} & \textbf{93.61} & \textbf{73.80} & \dposb{+11.34} \\
$\Delta_{\text{CED}}$
& \dpos{+4.00} & \dpos{+13.90} & \dpos{+20.17} & \dpos{+0.20} & \dpos{+4.39}
& \dpos{+10.35} & \dpos{+9.14} & \dpos{+39.67} & \dpos{+11.37} & \dposb{+12.58} & -- \\
\bottomrule
\end{tabular}%
}
\begin{flushleft}
\vspace{-1mm}
\footnotesize $^\dagger$ PAPO's global-perturbation KL reward induces repetitive reasoning loops on Qwen3.5-9B (25\% repeated-segment rate, 58\% budget exhaustion at 24k tokens).
\end{flushleft}
\vspace{-2mm}
\end{table}

\begin{figure}[t]
    \centering
    \includegraphics[width=\linewidth]{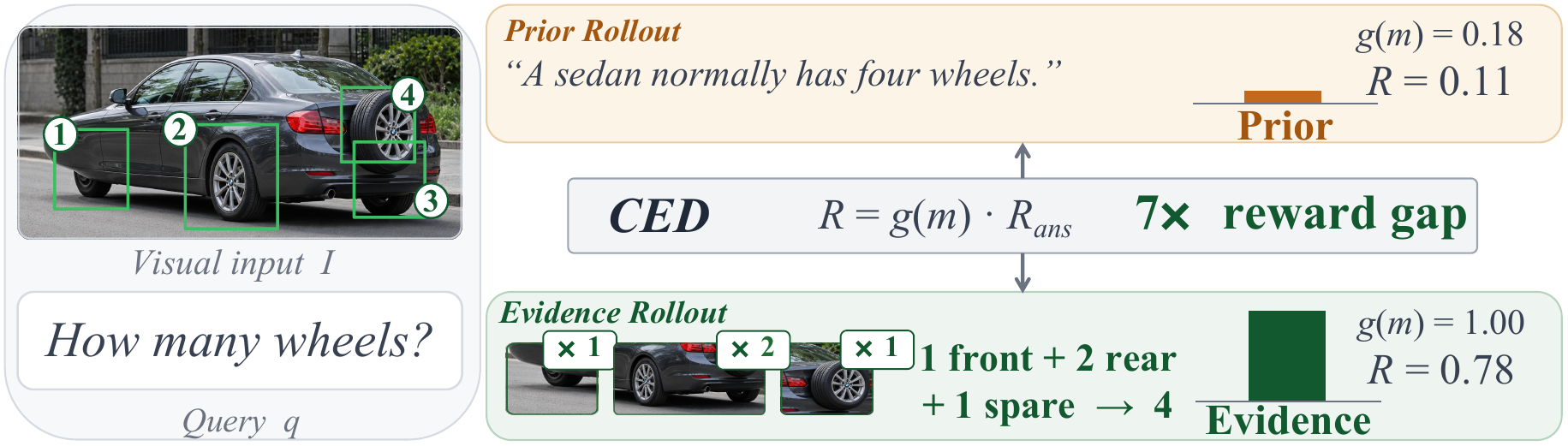}
    \vspace{-5mm}
    \caption{
A case showing why correctness-only reward is insufficient.
Two rollouts from the same GRPO group both answer ``4'' correctly,
but the \ced{} gate $g(m)$ creates a $7\times$ reward gap
between a prior-based rollout ($g(m){=}0.18$, $R{=}0.11$)
and an evidence-grounded rollout ($g(m){=}1.00$, $R{=}0.78$),
enabling GRPO to select the grounded trajectory.}
    \label{fig:worked-example}
\end{figure}

\subsubsection{Generalization and Variant Analysis}
\label{sec:generalization-variant}

\paragraph{Does CED transfer across backbones?}
\label{sec:cross-backbone}
We apply the same recipe to four from two model families. \Cref{tab:cross-backbone-ced} reports frozen-base scores, ours, and per-benchmark gains. Our method yields positive mean improvement on every backbone and non-negative gains on all 36 benchmark--backbone cells.

\begin{table}[t]
\centering
\scriptsize
\setlength{\tabcolsep}{3.4pt}
\renewcommand{\arraystretch}{0.84}
\setlength{\abovecaptionskip}{2pt}
\setlength{\belowcaptionskip}{3pt}
\caption{Cross-backbone validation of Ours. Each block reports the frozen base, Ours, and absolute $\Delta$ against that block's base. The Qwen3.5-9B block additionally reports 3-seed mean$\pm$std, with its mean $\Delta$ in the last column.}
\vspace{-2pt}
\label{tab:cross-backbone-ced}
\resizebox{\linewidth}{!}{%
\begin{tabular}{@{}l ccccc cccc c@{}}
\toprule
& \multicolumn{5}{c}{\textbf{Grounding}}
& \multicolumn{4}{c}{\textbf{General Reasoning}} & \\
\cmidrule(lr){2-6} \cmidrule(lr){7-10}
\textbf{Model}
& CountBench & SpatialEval & Hallusion & \shortstack{VLMs\\AreBlind} & FREAK
& MathVista & MMBench & MMMU & ScienceQA & \textbf{Avg.\ $\Delta$} \\
\midrule

\multicolumn{11}{l}{\emph{Backbone: Qwen2.5-VL-3B}} \\
Qwen2.5-VL-3B
  & \textbf{70.71} & 53.94 & 65.37 & 42.46 & 11.14
  & 52.10 & 80.40 & 50.29 & 74.93 & -- \\
\textbf{Ours}
  & \textbf{70.71} & \textbf{55.56} & \textbf{67.40} & \textbf{49.51} & \textbf{25.51}
  & \textbf{60.40} & \textbf{82.55} & \textbf{50.57} & \textbf{80.43} & -- \\
\quad $\Delta$
  & \dpos{+0.00} & \dpos{+1.62} & \dpos{+2.03} & \dpos{+7.05} & \dpos{+14.37}
  & \dpos{+8.30} & \dpos{+2.15} & \dpos{+0.28} & \dpos{+5.50}
  & \dposb{+4.59} \\

\midrule
\multicolumn{11}{l}{\emph{Backbone: Qwen2.5-VL-7B}} \\
Qwen2.5-VL-7B
  & 81.82 & 59.21 & 68.64 & 46.44 & 12.18
  & 63.10 & 84.73 & 50.63 & 83.98 & -- \\
\textbf{Ours}
  & \textbf{88.89} & \textbf{63.34} & \textbf{70.08} & \textbf{54.02} & \textbf{26.40}
  & \textbf{69.40} & \textbf{87.53} & \textbf{57.47} & \textbf{87.08} & -- \\
\quad $\Delta$
  & \dpos{+7.07} & \dpos{+4.13} & \dpos{+1.44} & \dpos{+7.58} & \dpos{+14.22}
  & \dpos{+6.30} & \dpos{+2.80} & \dpos{+6.84} & \dpos{+3.10}
  & \dposb{+5.94} \\

\midrule
\multicolumn{11}{l}{\emph{Backbone: Qwen3-VL-8B-Instruct}} \\
Qwen3-VL-8B
  & 94.90 & 65.26 & 72.86 & 67.01 & 28.85
  & 67.54 & 89.15 & 57.82 & 91.99 & -- \\
\textbf{Ours}
  & \textbf{95.92} & \textbf{67.75} & \textbf{74.30} & \textbf{68.88} & \textbf{31.18}
  & \textbf{68.04} & \textbf{89.58} & \textbf{58.85} & \textbf{92.82} & -- \\
\quad $\Delta$
  & \dpos{+1.02} & \dpos{+2.49} & \dpos{+1.44} & \dpos{+1.87} & \dpos{+2.33}
  & \dpos{+0.50} & \dpos{+0.43} & \dpos{+1.03} & \dpos{+0.83}
  & \dposb{+1.33} \\

\midrule
\multicolumn{11}{l}{\emph{Backbone: Qwen3.5-9B}} \\
Qwen3.5-9B
  & 80.80 & 52.73 & 47.20 & 46.44 & 9.95
  & 80.00 & 88.03 & 66.21 & 90.78 & -- \\
\textbf{Ours}
  & \textbf{93.90} & \textbf{54.35} & \textbf{78.10} & \textbf{73.10} & \textbf{23.12}
  & \textbf{83.75} & \textbf{88.58} & \textbf{75.67} & \textbf{93.61} & -- \\
\quad $\Delta$
  & \dpos{+13.10} & \dpos{+1.62} & \dpos{+30.90} & \dpos{+26.66} & \dpos{+13.17}
  & \dpos{+3.75} & \dpos{+0.55} & \dpos{+9.46} & \dpos{+2.83}
  & \dposb{+11.34} \\
\textbf{Ours (3 seeds)}
  & \shortstack{\textbf{93.60}\\$\pm0.48$}
  & \shortstack{\textbf{56.04}\\$\pm0.27$}
  & \shortstack{\textbf{78.35}\\$\pm0.21$}
  & \shortstack{\textbf{73.37}\\$\pm0.17$}
  & \shortstack{\textbf{23.11}\\$\pm0.40$}
  & \shortstack{\textbf{84.40}\\$\pm0.50$}
  & \shortstack{\textbf{89.87}\\$\pm0.04$}
  & \shortstack{\textbf{75.39}\\$\pm0.22$}
  & \shortstack{\textbf{94.46}\\$\pm0.13$}
  & \dposb{+11.83} \\
\bottomrule
\end{tabular}%
}
\end{table}

\paragraph{Which CED variant should be used: Answer or CoT?}
The default Answer variant directly rewards image-to-answer dependence. To isolate the effect of where the counterfactual margin is applied, \Cref{tab:answer-cot-compare} compares Answer and CoT variants on the same Qwen3.5-9B base. Answer obtains a slightly higher mean gain with stronger general transfer on downstream tasks like MMMU even though CoT obtains larger reward. This discrepancy reflects a structural mismatch between dense per-token rewards and GRPO's group-relative optimization. CoT-CED averages evidence margins over the full reasoning trace, and because trace length is policy-controlled, the model can raise the per-token average by truncating low-margin connective tokens rather than by strengthening visual grounding which is a reward-hacking pathway also documented in the process-reward literature \cite{guo2025deepseek, shao2024deepseekmath}. Tables~A.13--A.14 confirm this collapse: CoT-CED's mean chain shrinks to 3.6 tokens on counting, degenerating to object-cue shorthand. Answer-CED avoids this by scoring only the final-answer span, a fixed-length outcome interface that eliminates the chain-length degree of freedom. Full length and generation diagnostics are in Appendix~A.11.

\begin{table}[t]
\centering
\scriptsize
\setlength{\tabcolsep}{3.4pt}
\renewcommand{\arraystretch}{0.84}
\setlength{\abovecaptionskip}{2pt}
\setlength{\belowcaptionskip}{3pt}
\caption{Answer vs.\ CoT variants on Qwen3.5-9B; Avg.\ $\Delta$ is the mean gain over the base.}
\label{tab:answer-cot-compare}
\resizebox{\linewidth}{!}{%
\begin{tabular}{@{}l ccccc cccc c@{}}
\toprule
& \multicolumn{5}{c}{\textbf{Grounding}}
& \multicolumn{4}{c}{\textbf{General Reasoning}} & \\
\cmidrule(lr){2-6} \cmidrule(lr){7-10}
\textbf{Variant}
& CountBench & SpatialEval & Hallusion & \shortstack{VLMs\\AreBlind} & FREAK
& MathVista & MMBench & MMMU & ScienceQA & \textbf{Avg.\ $\Delta$} \\
\midrule
Answer
  & \textbf{93.90} & \textbf{54.35} & \textbf{78.10} & \textbf{73.10} & 23.12
  & \textbf{83.75} & \textbf{88.58} & \textbf{75.67} & 93.61
  & \dposb{+11.34} \\
CoT
  & 92.90 & 53.80 & \textbf{78.10} & 69.09 & \textbf{30.35}
  & 81.88 & 88.53 & 68.23 & \textbf{94.66}
  & \dpos{+10.60} \\
\bottomrule
\end{tabular}%
}
\end{table}

\Cref{fig:training-dynamics} shows why the default remains Answer-CED despite CoT-CED's stronger training-side signal. CoT-CED reaches higher reward and a denser positive margin during training, but this margin advantage does not yield stronger average transfer in \Cref{tab:answer-cot-compare}. The likely failure mode is that CoT-CED can raise a per-token margin by shortening the chain around visually salient object cues, whereas Answer-CED applies the counterfactual directly to the final answer. We therefore use Answer-CED as the default variant; additional length and generation diagnostics are in \Cref{app:cot-vs-answer}.

\begin{figure}[t]
    \centering
    \setlength{\abovecaptionskip}{3pt}
    \setlength{\belowcaptionskip}{4pt}
    \begin{minipage}[t]{0.48\linewidth}
        \centering
        \includegraphics[width=\linewidth]{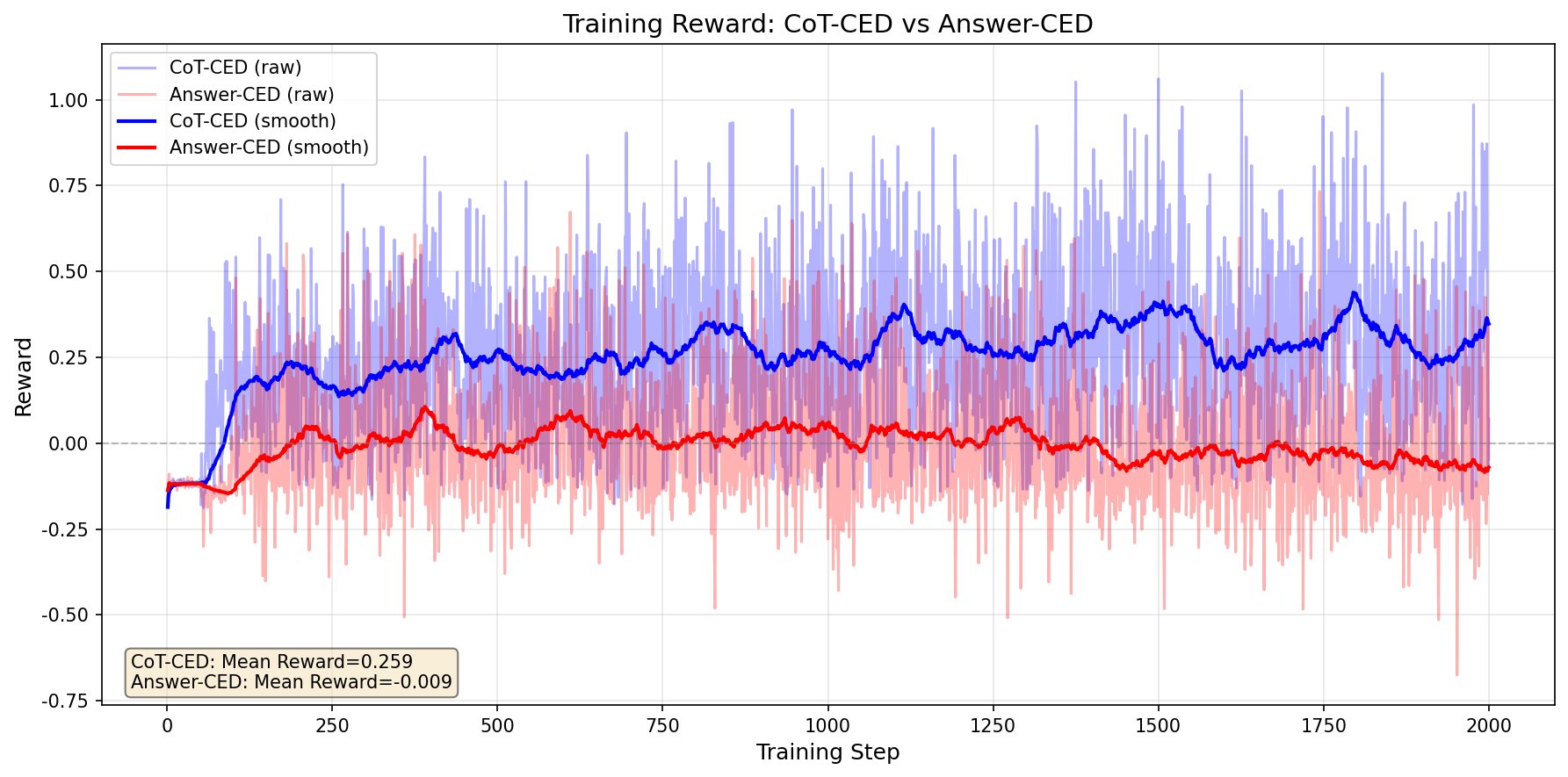}
        \par\vspace{1pt}
        {\footnotesize (a) Training reward}
    \end{minipage}
    \hfill
    \begin{minipage}[t]{0.48\linewidth}
        \centering
        \includegraphics[width=\linewidth]{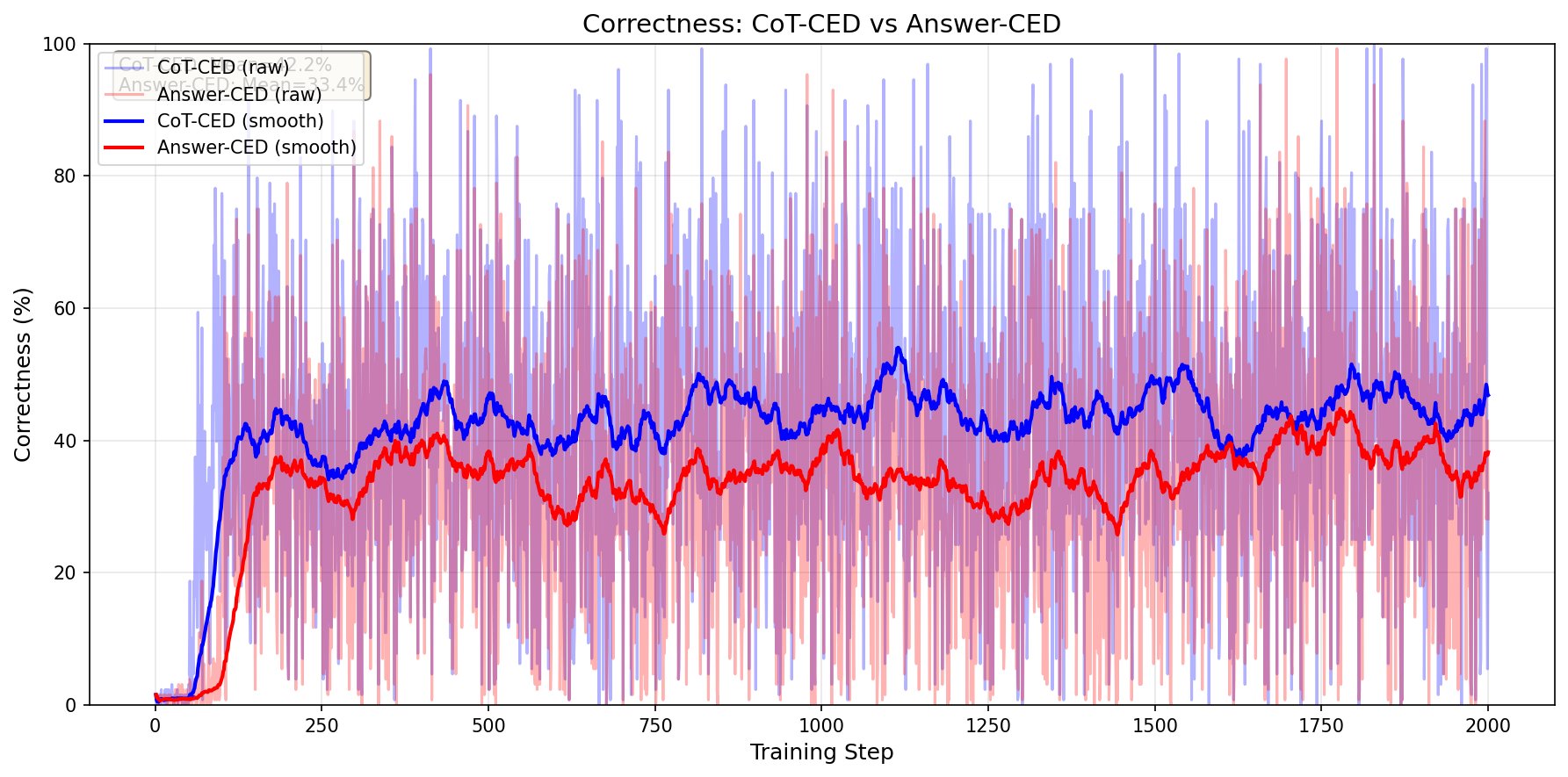}
        \par\vspace{1pt}
        {\footnotesize (b) Training correctness}
    \end{minipage}
    \caption{Training dynamics for Answer-CED and CoT-CED on Qwen3.5-9B. CoT-CED produces larger training-side rewards, but the downstream comparison in \Cref{tab:answer-cot-compare} favors Answer-CED on average; we therefore use Answer-CED as the default variant.}
    \label{fig:training-dynamics}
\end{figure}

\subsubsection{Reward Signal Robustness}
\label{sec:signal-robustness}

\paragraph{Is the reward signal robust to proposal quality?}
\label{sec:proposal-diagnostic}
\Cref{fig:proposal-robustness-main} evaluates whether \ced{}'s signal depends on proposal relevance or simply on masking magnitude. Panel~(a) compares the original COCO-based proposal with a random object box as $\Omega^{\mathrm{ev}}$ across seven task types. The evidence margin drops from $\bar m{=}0.268$ (COCO) to $\bar m{=}0.015$ (random box), and the relevant-beats-random rate drops from $0.636$ to $0.487$ (chance level). Panel~(b) shows IoU-graded degradation on the count-exclusion diagnostic: $\bar m$ decreases monotonically from $0.501$ at high IoU to $0.219$ at low IoU, approaching the random floor at $0.147$. Scale and shift perturbations are reported in \Cref{app:proposal-robustness}. These results confirm two properties: \ced{} does not reward masking magnitude, and the COCO-based proposal, despite being a weak spatial prior without question-specific annotation, provides sufficient evidence localization for a discriminative signal.

\begin{figure}[!htbp]
    \centering
    \includegraphics[width=\linewidth]{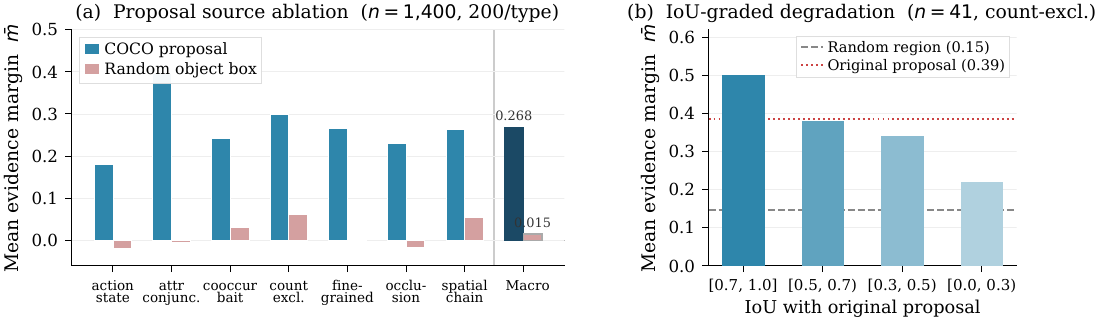}
    \vspace{-5mm}
    \caption{\ced{} probe robustness under proposal perturbation. \textbf{(a)}~Replacing the COCO proposal with a random object box as $\Omega^{\mathrm{ev}}$ collapses the evidence margin to near zero across all seven task types, confirming that the signal tracks proposal relevance rather than masking magnitude. \textbf{(b)}~IoU-graded degradation on the count-exclusion diagnostic: $\bar m$ decreases monotonically as spatial overlap with the original proposal drops, approaching the random-region floor at lowest IoU.}
    \label{fig:proposal-robustness-main}
\end{figure}

\paragraph{Which intervention type produces the strongest discrimination?}
\label{sec:perturbation-ablation}
\label{sec:overhead}
\label{sec:boundary}
Within local interventions, mean replacement gives the strongest discrimination, followed by zero replacement and Gaussian-noise replacement (AUC $0.669 > 0.641 > 0.629$ averaged across keying modes). Global random masking performs below chance, with an above-baseline rate of $0.490$, showing that ours depends on structured local counterfactuals rather than corruption magnitude alone. With the frozen pre-RL checkpoint, the Evidence Region signal is $1.534$, compared with $0.520$ for shifted regions, $0.319$ for wrong-category regions, and $0.566$ for random regions, consistent with the proposal-robustness results in \Cref{sec:proposal-diagnostic}. Full perturbation results and mechanistic analysis are in \Cref{app:perturbation-ablation,app:why-mean}.

\section{Discussion}
\label{sec:discussion}
\label{sec:self-evolve}

\paragraph{Visual evidence over language priors on FREAK.}
The FREAK gains make the same mechanism visible at the trajectory level. The representative cases in \Cref{fig:qualitative-prior-cases} contain images whose visual evidence contradicts a common language prior. PAPO and LLaVA-v1.6-7B follow the prior-consistent answer, whereas the \ced{}-trained model recovers the image-grounded answer. Thus, the qualitative rollouts instantiate the rerouting predicted by \Cref{fig:causal-graph}: answers that would otherwise be driven by shortcut priors are redirected toward visual evidence. The selected cases are typical members of this failure regime rather than engineered edge cases.

\begin{figure}[t]
\centering
\setlength{\tabcolsep}{1.3pt}
\renewcommand{\arraystretch}{0.96}
\newcommand{\freakcase}[6]{%
\begin{minipage}[t]{0.192\linewidth}
\centering
\includegraphics[height=0.88in,width=0.98\linewidth,keepaspectratio]{#1}\par\vspace{1.4pt}
{\scriptsize\bfseries #2}\par\vspace{1.4pt}
\begin{minipage}{0.98\linewidth}
\scriptsize\raggedright
\textit{Q:} #3\\[0.5pt]
\textit{GT:} #4\\[0.5pt]
\textit{Other:} #5\\[0.5pt]
\textit{ours:} \textbf{#6}
\end{minipage}
\end{minipage}%
}
\begin{tabular}{@{}ccccc@{}}
\freakcase{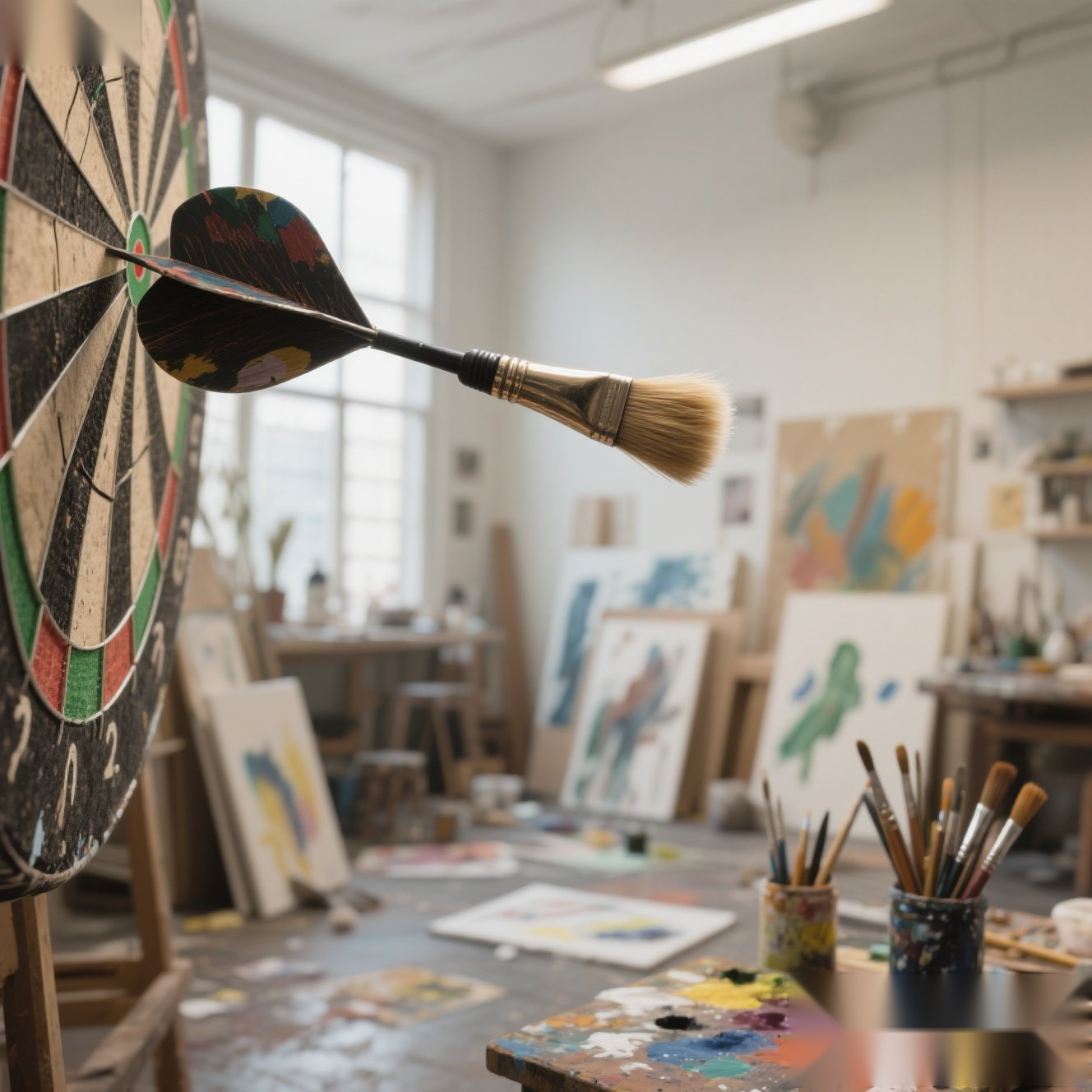}{FREAK-548}{point or tail?}{tail}{point/point/head}{tail} &
\freakcase{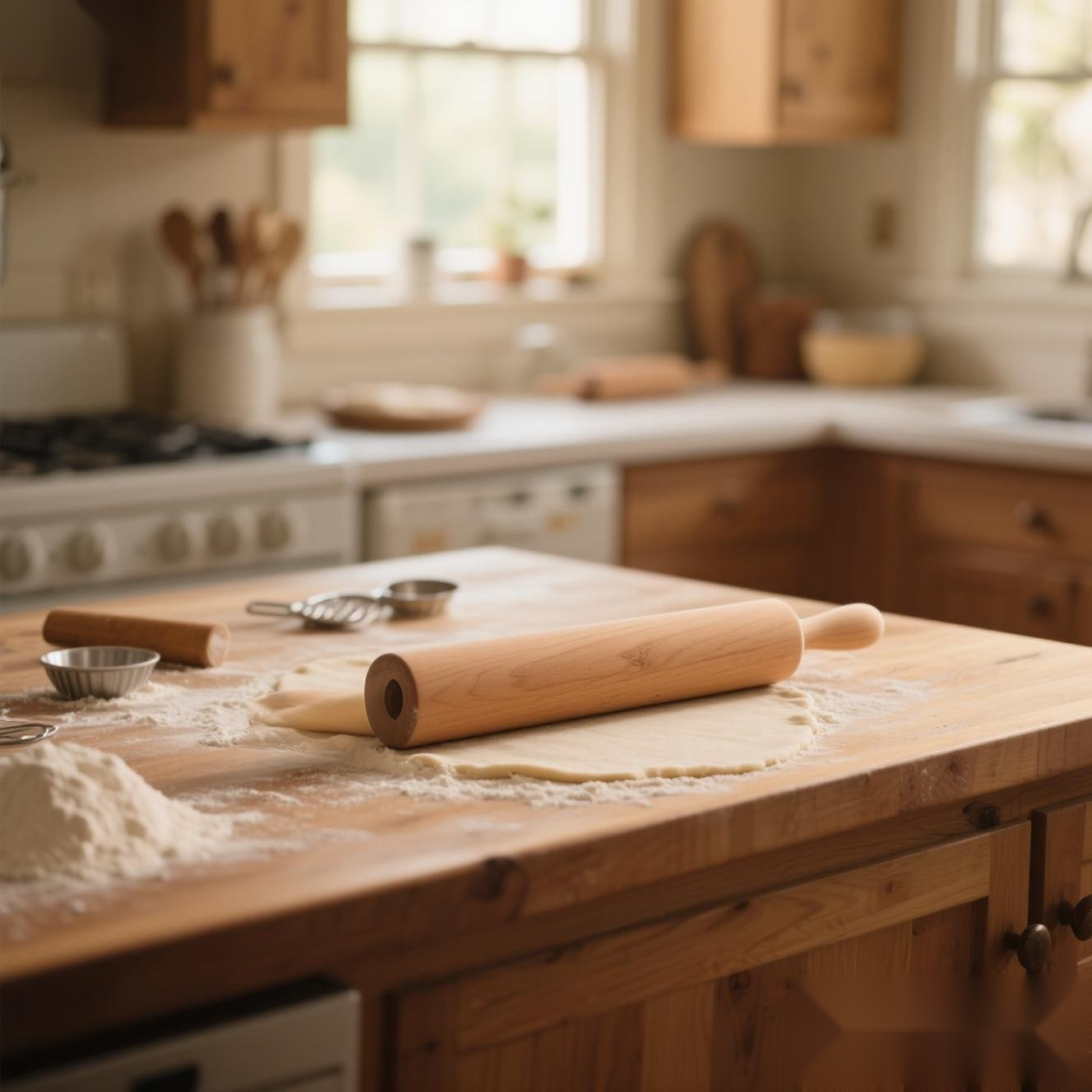}{FREAK-1683}{rolling-pin handles?}{1}{2/2/2}{1} &
\freakcase{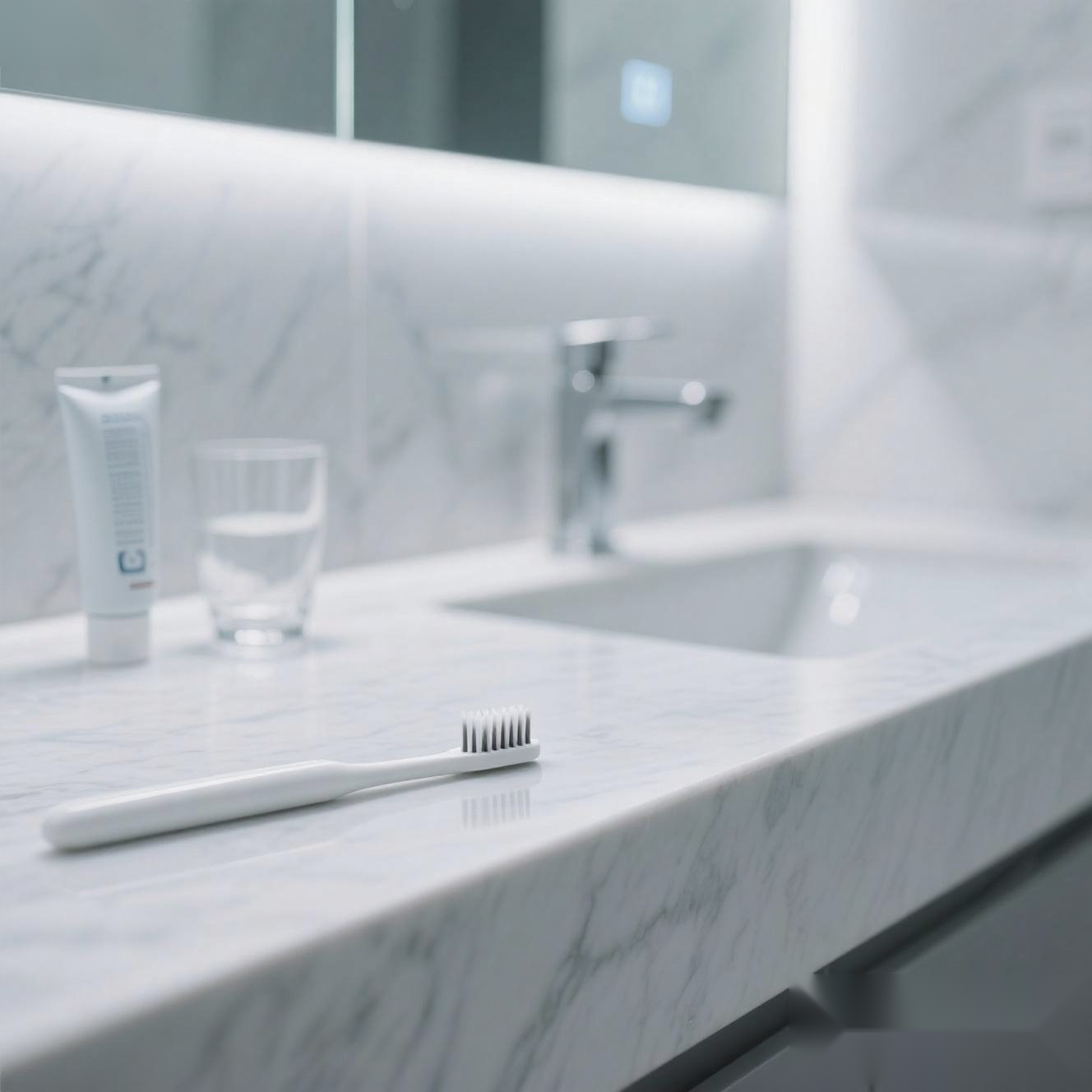}{FREAK-1694}{toothbrush rows?}{1}{3/2/3}{1} &
\freakcase{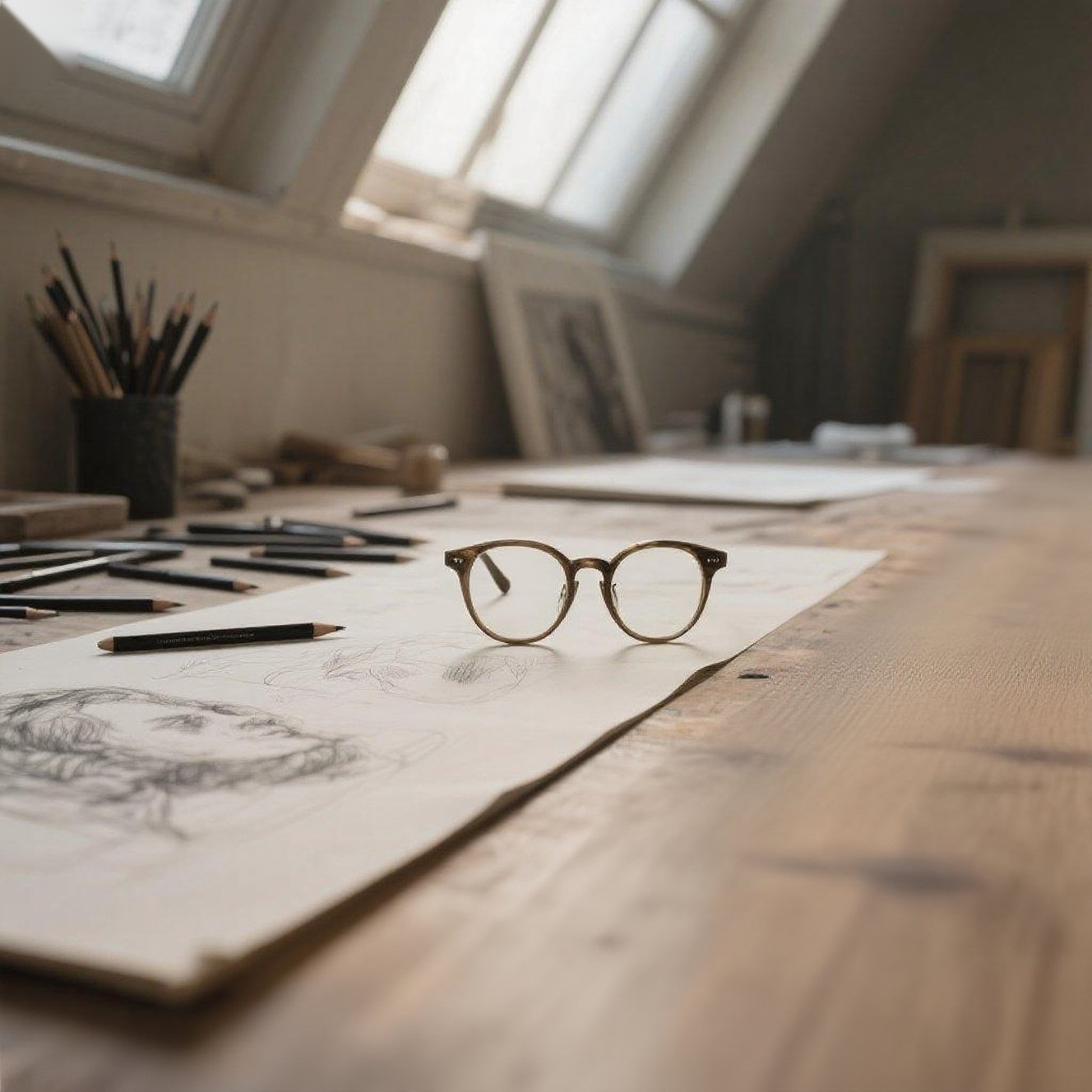}{FREAK-1713}{eyeglass arms?}{1}{2/2/2}{1} &
\freakcase{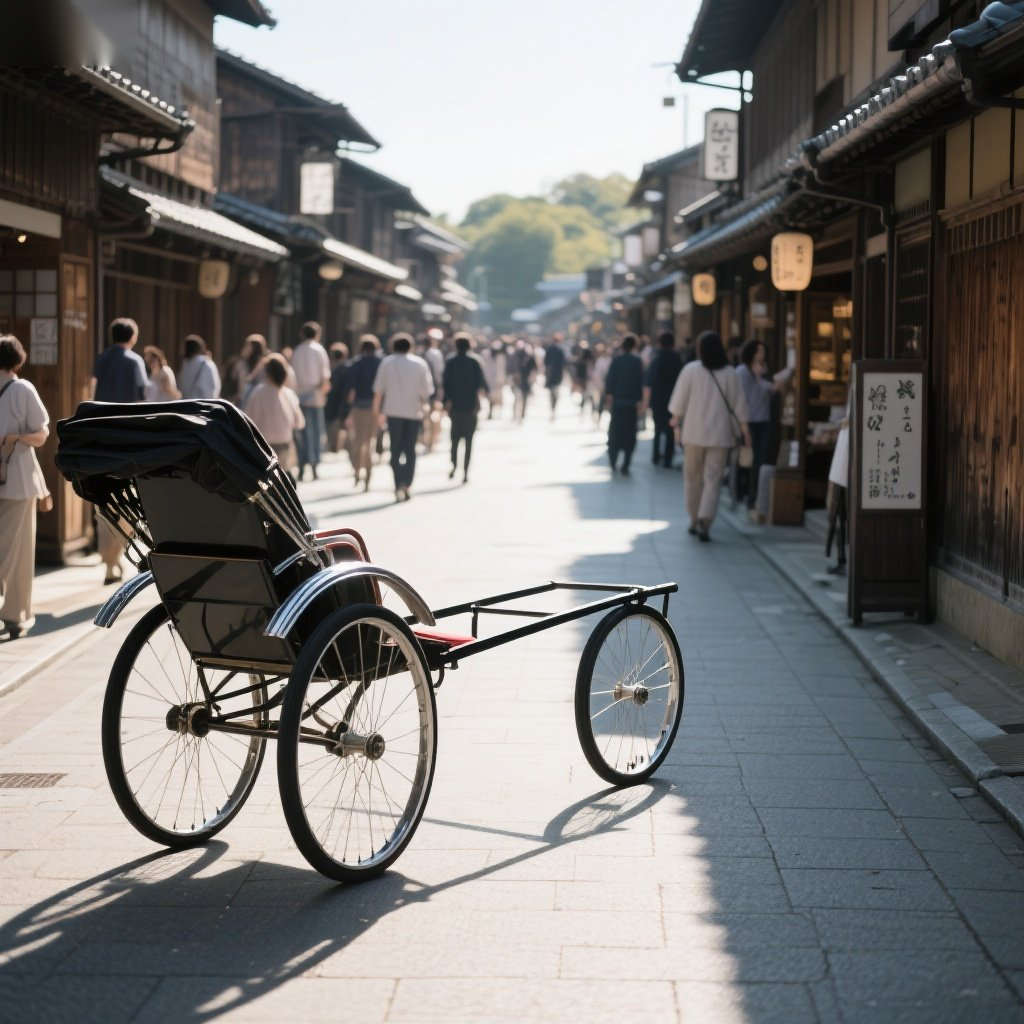}{FREAK-1751}{rickshaw tires?}{3}{2/2/2}{3}
\end{tabular}
\caption{Five prior-conflict FREAK cases. Each card reports the ground truth, the outputs of PAPO, LLaVA-v1.6-7B, and the Qwen2.5-VL-7B base model in that order, and the output of ours. The non-CED models follow a plausible object prior, whereas ours follows the local visual evidence.}
\label{fig:qualitative-prior-cases}
\end{figure}

\paragraph{Behavior in the absence of visual evidence.}
Blank-image ScienceQA items provide the complementary test. In examples whose image is only a blank $448{\times}448$ placeholder, such as ``What is the mass of a cement truck?'' with options 20 tons / 20 pounds / 20 ounces, the frozen base selects the prior-consistent answer and is scored correct, while \ced{} responds ``Cannot determine from image'' and is scored incorrect. This is precisely the downstream tasks' behavior targeted by the reward: when no answer is grounded in image evidence, the model should avoid substituting language priors. This also explains part of the modest ScienceQA gains in \Cref{tab:cross-backbone-ced}: \ced{} improves the visual-reasoning portion as in other benchmarks, while the blank-image subset rewards the prior-following behavior that \ced{} suppresses. The net effect remains positive on every backbone.

\paragraph{No text-only degradation.}
A potential concern is that rewarding evidence dependence may harm general text reasoning. We evaluate the Qwen3.5-9B Answer-CED checkpoint on eight text-only benchmarks (ARC-Easy, ARC-Challenge, OpenBookQA, CommonsenseQA, WinoGrande, PIQA, BoolQ, TruthfulQA-MC2) spanning commonsense, science, and factual reasoning. The mean accuracy change is $-0.27$~pp, with no single benchmark dropping more than $1.5$~pp. \ced{} improves visual grounding without degrading the language backbone.

\paragraph{Evidence-closed text-only self-evolution.}
The following analysis makes this intuition precise: an image-conditional signal is structurally necessary to distinguish grounded from shortcut trajectories. In VLM RL, model-judged or self-evolving audits often collapse back to text because the cheap verifiable targets used in LLM self-evolution, such as executable code or checkable arithmetic, have no per-sample-cheap visual analogue. Let $S_t$ be the reward, $I_t$ the image, $\hat Y_t$ the answer, and $G_t \in \{0,1\}$ the latent grounding indicator. This conditional independence is a direct consequence of the path structure in Figure 1: the evidence path and the shortcut path both terminate at the answer node, so a reward that observes only the answer cannot determine which upstream path produced it.A text-only reward satisfies
$S_t \perp\!\!\!\perp I_t \mid (\hat Y_t, Q_t)$,
and therefore
$I(S_t;\, G_t \mid \hat Y_t,\, Q_t) = 0$:
two trajectories that reach the same correct answer through grounding versus a language-prior shortcut induce identical reward distributions. We call this regime \emph{evidence-closed self-evolution}. As long as the audit observes only text, judge capacity, annealing, or trustworthiness reweighting cannot resolve the ambiguity, since the limitation comes from the conditional independence structure. \ced{} breaks this closure by injecting an image-conditional signal. The evidence margin $m(I,q,y)$ in \Cref{eq:margin} cannot be computed from $(\hat Y_t, Q_t)$ alone, and the strict inequality
$H(G_t \mid \mathcal{B}_t, A_t) < H(G_t \mid \mathcal{B}_t)$
formalizes the empirical fact that \ced{} reorders trajectories that correctness rewards cannot. The formal argument is developed in \Cref{app:information-theory}. This yields a unified diagnosis of recent self-evolving multimodal RL pipelines: stable shortcut correction depends less on judge capacity than on whether the reward can observe the visual evidence path.

\section{Conclusion}
\label{sec:conclusion}

We introduced \sysname{}, a post-training framework for VLMs built around \ced{}, a counterfactual evidence diagnostic motivated by a causal view of visual grounding. \ced{} moves beyond rewarding answer correctness alone by asking whether a sampled response depends on the Evidence Region causally. With spatial-region interventions, this diagnostic can be directly integrated into GRPO as an evidence-aware reward, producing grounded reasoning behavior with no additional inference cost. Across four backbones and nine benchmarks, \sysname{} consistently improves performance, with the largest gains on benchmarks where shortcut behavior is most exposed.

Post-training should reward correct answers for depending on the visual evidence relevant to the question, rather than for matching the answer distribution alone. Though our experiments use object-centric interventions, gains generalize to various downstream tasks: the \ced{} signal supervises the reasoning process and instills faithful evidence dependence, the capability most absent from current VLMs. Achieving high accuracy without depending on the relevant evidence is not visual reasoning; making that dependence a prioritized training objective is what evidence-intensive reasoning requires.

\bibliographystyle{plainnat}
\bibliography{refs}

\clearpage
\appendix
\setcounter{figure}{0}
\setcounter{table}{0}
\renewcommand{\thefigure}{A.\arabic{figure}}
\renewcommand{\thetable}{A.\arabic{table}}
\renewcommand{\theHfigure}{appendix.figure.\arabic{figure}}
\renewcommand{\theHtable}{appendix.table.\arabic{table}}
\providecommand{\FloatBarrier}{}

\section{Supplementary Experiments and Implementation Details}

\subsection{Reproducibility Details}
\label{app:reproducibility}

\paragraph{Dataset split and task-family distribution.}
The training data are derived from COCO val2017~\citep{lin2014microsoft} after quality filtering and format normalization.
This yields 21,758 raw samples, split into train (17,502 valid), validation (2,352 valid), and probe (1,904 valid) subsets.
After RL-specific filtering, 15,314 samples are loaded for online training.
The loaded training set contains three task families: attribute (6,562), counting (2,188), and spatial (6,564).
Training images are drawn from COCO val2017~\citep{lin2014microsoft}. CountBench~\citep{paiss2023teaching}, SpatialEval~\citep{wang2024picture}, and MathVista~\citep{lu2023mathvista} use independently sourced evaluation images and share no image-level overlap with the training set; the remaining benchmarks are cited in \Cref{sec:experiments} and likewise contain no COCO val2017 images.

\paragraph{Training configuration.}
Training uses \texttt{balanced\_no\_replacement} sampling for uniform coverage across task families, GRPO group size $32$, and $2{,}000$ optimization steps.
The trainable part of the model is restricted to the last four layers.
Full training and reward hyperparameters are listed in \Cref{tab:hyperparams}.

\begin{table}[t]
\centering
\small
\setlength{\tabcolsep}{5pt}
\caption{Complete hyperparameter listing. ``Paper notation'' gives the corresponding symbol in the main text when applicable.}
\label{tab:hyperparams}
\resizebox{\linewidth}{!}{%
\begin{tabular}{@{}llcl@{}}
\toprule
\textbf{Category} & \textbf{Parameter} & \textbf{Value} & \textbf{Paper notation} \\
\midrule
\multicolumn{4}{l}{\emph{RL \& optimization}} \\
Learning rate              & \texttt{lr}                    & $1 \times 10^{-5}$ & --- \\
KL penalty coefficient     & \texttt{kl\_coeff}             & 0.01   & --- \\
GRPO group size            & \texttt{group\_size}           & 32     & --- \\
Training steps             & \texttt{n\_steps}              & 2{,}000 & --- \\
Trainable layers           & \texttt{n\_trainable\_layers}  & 4 (last) & --- \\
Precision                  & \texttt{dtype}                 & bfloat16 & --- \\
\midrule
\multicolumn{4}{l}{\emph{Generation \& sampling}} \\
Temperature                & \texttt{temperature}           & 1.0    & --- \\
Top-$p$                    & \texttt{top\_p}                & 0.95   & --- \\
Max new tokens (evidence)  & \texttt{max\_new\_tokens}      & 32     & --- \\
Max response tokens (full) & \texttt{max\_response\_tokens} & 64     & --- \\
Prompt format              & \texttt{short\_evidence\_v1}   & ---    & --- \\
\midrule
\multicolumn{4}{l}{\emph{Reward function (\Cref{eq:reward})}} \\
Response reward weight     & \texttt{alpha\_resp}           & 0.70   & see note$^\dagger$ \\
Answer reward weight       & \texttt{alpha\_ans}            & 0.30   & see note$^\dagger$ \\
Response logprob temp.     & \texttt{tau\_resp}             & 0.20   & --- \\
Answer logprob temp.       & \texttt{tau\_ans}              & 1.00   & --- \\
Gate temperature           & \texttt{tau\_resp} (shared)    & 0.20   & $\tau_g$ \\
Gate baseline at $m{=}0$   & (derived)                     & 0.50   & $g(0)$ \\
Evidence tie-breaker       & \texttt{evidence\_eps}         & 0.10   & $\varepsilon_{\text{tie}}$ \\
Additive evidence weight   & \texttt{evidence\_eps} (shared)& 0.10  & $\lambda$ \\
Reward clipping            & \texttt{[min, max]}            & $[-1.25,\; 1.00]$ & --- \\
Format penalty             & ---                            & $-0.75$ & --- \\
\midrule
\multicolumn{4}{l}{\emph{Data \& evaluation}} \\
Rerank candidates          & \texttt{rerank\_num\_candidates} & 6    & --- \\
non-evidence Regions per sample & $K$                          & 3      & $K$ \\
\bottomrule
\end{tabular}%
}
\vspace{2pt}

\raggedright\footnotesize
$^\dagger$ In \Cref{eq:reward}, $R_{\text{ans}}$ is a composite:
$R_{\text{ans}} = \texttt{alpha\_ans} \cdot \text{answer\_score} + \texttt{alpha\_resp} \cdot \text{response\_score}$,
where each score is computed from log probabilities scaled by its temperature.
\end{table}

\paragraph{Broader impacts.}
\ced{} is a training-time diagnostic and reward component for improving visual grounding in VLMs.
It does not directly interact with end users or generate user-facing content.
Improved grounding may reduce shortcut-driven errors in safety-relevant visual reasoning tasks.
The main risk is that counterfactual probing can also reveal model vulnerabilities, although \ced{} exposes only a scalar evidence margin rather than a detailed attack surface.

\subsection{Intervention Implementation Details}
\label{app:intervention-details}

\paragraph{Replacement strategies.}
We evaluate three replacement strategies: \emph{mean replacement} (default), \emph{zero replacement}, and \emph{Gaussian-noise replacement}.
All interventions are applied to the last-layer vision-encoder output after spatial merging.
Mean replacement substitutes Evidence Region visual tokens with the mean of neighboring visual tokens outside $\mathcal{T}(\Omega)$; zero replacement sets selected tokens to $\mathbf{0}$; Gaussian-noise replacement substitutes zero-mean Gaussian samples.

\paragraph{Spatial-merge-aware token mapping.}
Modern vision encoders apply a $2{\times}2$ spatial merge before feeding visual tokens to the LLM.
Given merge dimensions $(s_w,s_h)=(2,2)$, an original grid coordinate $(g_x,g_y)$ maps to the merged-token index
\[
(g_y \mathbin{/\mkern-4mu/} s_h) \times w_{\mathrm{merged}} + (g_x \mathbin{/\mkern-4mu/} s_w).
\]
For non-evidence Region construction, the mapping module supports Chebyshev-distance ring expansion on the merged feature map.

\paragraph{Proposal construction pipeline.}
Positive proposals are sourced from COCO instance annotations~\citep{lin2014microsoft}.
They are object-level bounding boxes, not question-specific region annotations.
The Evidence Region $\Omega^{\mathrm{ev}}$ is resolved from sample metadata using the priority
\texttt{proposal\_bbox} $>$ \texttt{target\_bbox} $>$ \texttt{argument\_bbox} $>$ \texttt{metadata.bbox} $>$ center-box fallback.
non-evidence Regions $\{\Omega^{\mathrm{non}}_k\}$ are sampled as area-matched random boxes with IoU $\le 0.05$ to $\Omega^{\mathrm{ev}}$.
If box sampling fails, random visual-token subsets of matching size are used as fallback.
Quality diagnostics filter proposals with extreme area fractions or insufficient token coverage before training.

\subsection{Comparison with Perception-Aware RL Methods}
\label{app:method-comparison}

\Cref{tab:method-comparison} compares \ced{} with representative perception-aware VLM RL methods.
The key distinction is whether the training signal tests candidate-answer dependence on local visual evidence, rather than coarse image dependence or text-only consistency.

\begin{table}[t]
\centering
\small
\setlength{\tabcolsep}{3pt}
\caption{Methodological comparison with perception-aware VLM RL methods. The \emph{Causal audit} column indicates whether the method tests causal evidence dependence for a candidate answer.}
\label{tab:method-comparison}
\resizebox{\linewidth}{!}{%
\begin{tabular}{@{}lcccc@{}}
\toprule
\textbf{Method} & \textbf{Signal source} & \textbf{Granularity} & \textbf{Causal audit} & \textbf{Stage} \\
\midrule
PAPO~\citep{wang2025papo} & Global mask KL & Global & No & Training \\
VPPO~\citep{huang2025vppo} & Attention weights & Token-level & No & Training \\
Vision-SR1\textsuperscript{\textdagger} & Text self-verification & Description-level & No & Training \\
PaLMR\textsuperscript{\textdaggerdbl} & LLM-as-Judge & Trajectory-level & No & Training \\
SophiaVL-R1~\citep{fan2025sophiavl} & Thinking-reward RM & Trajectory-level & No & Training \\
PEARL\textsuperscript{\S} & Perception checklist & Sample-level & No & Training \\
VAPO\textsuperscript{\P} & Trajectory anchoring & Token-level & No & Training \\
VLM-R1\textsuperscript{*} & Self-verification & Response-level & No & Training \\
Perception-R1\textsuperscript{\textdagger\textdagger} & Proxy localization & Region-level$^a$ & No & Training \\
\midrule
\ced{} (Ours) & Counterfactual intervention & \textbf{Region-level} & \textbf{Yes} & Training \\
\bottomrule
\end{tabular}%
}
\begin{flushleft}
\footnotesize
\textsuperscript{\textdagger}\citet{li2025self}; \textsuperscript{\textdaggerdbl}\citet{li2026palmr}; \textsuperscript{\S}\citet{zhang2025perceptual}; \textsuperscript{\P}\citet{yue2025vapo}; \textsuperscript{*}\citet{shen2025vlm}; \textsuperscript{\textdagger\textdagger}\citet{xiao2025perceptionr1}.\\
$^a$Perception-R1 reaches region-level granularity through proxy localization and requires additional supervision.
\end{flushleft}
\end{table}

\subsection{Additional Evidence-Signal Diagnostics}
\label{app:signal-extended}
\label{app:sanity-checks}

\paragraph{Relevant-vs-random intervention.}
On a held-out set of 500 samples per task family, we compare Evidence Region sensitivity $s(\Omega^{\mathrm{ev}})$ with mean non-evidence Region sensitivity $\mu(s^{\mathrm{non}})$.
On counting tasks, the Evidence Region elicits higher sensitivity in 63.4\% of samples, with mean Evidence Region $s=2.566$ versus mean non-evidence Region $s=0.638$.
On presence tasks, the rate drops to 55.0\%, with mean Evidence Region $s=1.082$ versus mean non-evidence Region $s=0.302$.
This supports the regime separation used by the routing rule in \Cref{sec:ced-in-grpo}.

\paragraph{Cross-model validation.}
We also validate the signal on InternVL3.5-8B~\citep{wang2025internvl3} and LLaVA-v1.6-Mistral-7B~\citep{li2024llava} using pixel-level intervention.
\Cref{tab:cross-model} shows that counting tasks maintain higher relevant-vs-random rates across both models, while presence tasks remain weaker.

\begin{table}[t]
\centering
\small
\setlength{\tabcolsep}{6pt}
\caption{Cross-model evidence sensitivity validation ($N{=}50$ per task per model).}
\label{tab:cross-model}
\begin{tabular*}{\linewidth}{@{\extracolsep{\fill}}llcc@{}}
\toprule
\textbf{Model} & \textbf{Task} & \textbf{Relevant $>$ Random rate} & \textbf{Mean margin} \\
\midrule
\multirow{2}{*}{InternVL3.5-8B} & Counting  & 0.68 & 0.0070 \\
                                      & Presence  & 0.62 & 0.0737 \\
\midrule
\multirow{2}{*}{LLaVA-v1.6-7B}   & Counting  & 0.78 & 0.0074 \\
                                      & Presence  & 0.60 & 0.1625 \\
\bottomrule
\end{tabular*}
\end{table}

\paragraph{Blindfold test.}
After replacing the image with a blank input, the relevant-vs-random rate drops to 0.404 on counting and 0.448 on presence, both below the 0.50 chance level.
Mean delta collapses from $+1.928/+0.780$ to $-0.114/+0.070$, confirming that the signal originates from real image content.

\begin{table}[t]
\centering
\small
\setlength{\tabcolsep}{6pt}
\caption{Blindfold test results ($N{=}500$ per task family).}
\label{tab:blindfold}
\begin{tabular*}{\linewidth}{@{\extracolsep{\fill}}llcc@{}}
\toprule
\textbf{Condition} & \textbf{Task} & \textbf{Rel. $>$ Rand. rate} & \textbf{Mean Delta} \\
\midrule
Normal    & Counting  & 0.634 & $+1.928$ \\
Normal    & Presence  & 0.550 & $+0.780$ \\
\midrule
Blindfold & Counting  & 0.404 & $-0.114$ \\
Blindfold & Presence  & 0.448 & $+0.070$ \\
\bottomrule
\end{tabular*}
\end{table}

\subsection{Intervention-Type Ablations}
\label{app:perturbation-ablation}
\label{app:why-mean}
\label{app:intervention-artifact}
\label{app:baseline-proxy}

\paragraph{Why mean replacement reduces intervention artifacts.}
The perturbation ablation shows that mean replacement gives the strongest discrimination within the local intervention family: AUC $0.669$ for mean replacement, compared with $0.641$ for zero replacement and $0.629$ for Gaussian-noise replacement.
This ordering follows from the role of the replacement as a structural baseline.
Write the Evidence Region token representation as
\begin{equation}
\label{eq:decomposition}
    h_T = b_T + e_T,
\end{equation}
where $b_T$ is a context-consistent baseline component and $e_T$ is the answer-relevant evidence component.
The ideal counterfactual removes $e_T$ while preserving $b_T$.
Replacement strategies differ in how much they perturb this baseline:
\begin{align}
\label{eq:three-replacements}
    h_T^{\mathrm{mean}}  &= \mu_T = b_T + \delta_T, &
    h_T^{\mathrm{zero}}  &= \mathbf{0}, &
    h_T^{\mathrm{noise}} &= b_T + \xi_T .
\end{align}
Let $F$ be the downstream logit assigned to the candidate answer and $J_T$ its local Jacobian with respect to $h_T$.
A first-order expansion gives
\begin{align}
\label{eq:three-delta}
    \Delta F^{\mathrm{mean}}  &\approx J_T(\delta_T - e_T), &
    \Delta F^{\mathrm{zero}}  &\approx -J_T(b_T + e_T), &
    \Delta F^{\mathrm{noise}} &\approx J_T(\xi_T - e_T).
\end{align}
All three contain the desired signal term $-J_T e_T$.
Mean replacement adds only the neighborhood-offset artifact $J_T\delta_T$, zero replacement adds the systematic baseline-removal artifact $-J_Tb_T$, and Gaussian replacement adds sample-specific noise.
Since context mixing makes the neighborhood offset smaller than the full baseline norm at the intervention layer, mean replacement preserves the non-evidence part of the representation more faithfully.

\paragraph{AUC ablation and global masking.}
\Cref{tab:gaussian-noise-ablation} reports the local perturbation-family ablation.
Global masking performs worse: it does not support stable answer-conditioned discrimination, with above-baseline rate $0.490$ below chance.
Local counterfactual intervention retains discriminative utility with AUC $0.620$.

\begin{table}[t]
\centering
\small
\setlength{\tabcolsep}{7pt}
\caption{Perturbation-family ablation within the local intervention family, evaluated by logits-JS AUC.}
\label{tab:gaussian-noise-ablation}
\begin{tabular*}{\linewidth}{@{\extracolsep{\fill}}lccc@{}}
\toprule
\textbf{Key mode} & \textbf{Mean repl.} & \textbf{Zero repl.} & \textbf{Gaussian-noise repl.} \\
\midrule
prompt\_last  & \textbf{0.6780} & 0.6602 & 0.6299 \\
answer\_first & \textbf{0.6592} & 0.6215 & 0.6273 \\
\midrule
Average       & \textbf{0.6686} & 0.6409 & 0.6286 \\
\bottomrule
\end{tabular*}
\end{table}

\begin{table}[t]
\centering
\small
\setlength{\tabcolsep}{6pt}
\caption{Gaussian-noise profile under \texttt{prompt\_last}. Stability drops under \texttt{logits\_only}.}
\label{tab:gaussian-noise-profile}
\begin{tabular*}{\linewidth}{@{\extracolsep{\fill}}lccc@{}}
\toprule
\textbf{Config} & \textbf{Fixed metric} & \textbf{Best metric} & \textbf{Best paired AUC} \\
\midrule
logits24     & 0.6299 & 0.6299 & 0.7108 \\
logits\_only & 0.5753 & 0.5920 & 0.6407 \\
\bottomrule
\end{tabular*}
\end{table}

\subsection{Reward Hyperparameter Sensitivity}
\label{app:hyperparameter-sweep}

\Cref{tab:lambda-sweep,tab:tau-g-sweep} show that the default reward hyperparameters lie in stable regions.
Across $\lambda\in[0.05,0.20]$, the non-constant reward rate remains 96--100\%.
For $\tau_g$, values around 0.15--0.30 yield balanced selection behavior; we use $\tau_g=0.20$ by default.

\begin{table}[t]
\centering
\small
\setlength{\tabcolsep}{5pt}
\caption{Evidence weight ($\lambda$) sensitivity. $\lambda=0.00$ is correctness-only.}
\label{tab:lambda-sweep}
\begin{tabular*}{\linewidth}{@{\extracolsep{\fill}}lccc@{}}
\toprule
$\lambda$ & Non-const. rate & Reward std & Zero-var. rate \\
\midrule
0.00 & 100\% & 0.0528 & 0\% \\
0.05 & 100\% & 0.0508 & 0\% \\
\textbf{0.10} & \textbf{96\%} & \textbf{0.0597} & \textbf{4\%} \\
0.15 & 100\% & 0.0415 & 4\% \\
0.20 & 100\% & 0.0491 & 0\% \\
\bottomrule
\end{tabular*}
\end{table}

\begin{table}[t]
\centering
\small
\setlength{\tabcolsep}{5pt}
\caption{Gate temperature ($\tau_g$) sensitivity.}
\label{tab:tau-g-sweep}
\begin{tabular*}{\linewidth}{@{\extracolsep{\fill}}lccc@{}}
\toprule
$\tau_g$ & Non-const. rate & Reward std & Zero-var. rate \\
\midrule
0.10 & 100\% & 0.0568 & 4\% \\
\textbf{0.20} & \textbf{96\%} & \textbf{0.0597} & \textbf{4\%} \\
0.30 & 96\% & 0.0469 & 4\% \\
0.50 & 100\% & 0.0378 & 0\% \\
\bottomrule
\end{tabular*}
\end{table}

\subsection{Proposal Robustness to Spatial Perturbations}
\label{app:proposal-robustness}

This experiment tests whether the \ced{} probe remains useful when the target proposal is imperfect.
It is orthogonal to \Cref{app:perturbation-ablation}: that subsection varies the intervention type while holding the region fixed; here we hold the intervention type fixed as mean replacement and vary the region.

\paragraph{Protocol.}
Starting from the true Evidence Region box of each sample, we generate perturbed proposals and re-run the full \ced{} pipeline with $K{=}3$ non-evidence Regions.
We test three perturbation axes: IoU stratification with the original box, scale changes around the box center, and translations by a fixed fraction of the image short side.
A \texttt{random} region baseline gives the expected floor without spatial targeting.
We report mean evidence margin $\bar m$, evidence-non-evidence gap $\Delta=\mu(s^{\mathrm{ev}})-\mu(s^{\mathrm{non}})$, and evidence-beats-random rate $\Pr(s^{\mathrm{ev}}>s^{\mathrm{rand}})$.
All numbers use the Qwen3.5-9B CoT-CED checkpoint on $n{=}41$ count-exclusion samples.

\begin{table}[t]
\centering
\small
\setlength{\tabcolsep}{6pt}
\caption{Proposal-perturbation results on count-exclusion ($n{=}41$). Higher values indicate stronger evidence dependence; the \texttt{random} row gives the floor without spatial targeting.}
\label{tab:proposal-robustness}
\begin{tabular*}{\linewidth}{@{\extracolsep{\fill}}llccc@{}}
\toprule
\textbf{Axis} & \textbf{Condition} & $\bar m$ & $\Delta$ & \textbf{Rel. rate} \\
\midrule
\multirow{4}{*}{IoU stratification}
  & $[0.7,\,1.0]$ & \textbf{0.501} & \textbf{2.032} & \textbf{0.732} \\
  & $[0.5,\,0.7)$ & 0.379 & 1.282 & 0.683 \\
  & $[0.3,\,0.5)$ & 0.341 & 1.386 & 0.683 \\
  & $[0.0,\,0.3)$ & 0.219 & 0.555 & 0.561 \\
\midrule
\multirow{5}{*}{Scale}
  & $0.50\times$ & 0.377 & 1.113 & 0.732 \\
  & $0.75\times$ & 0.481 & 1.750 & 0.732 \\
  & $1.25\times$ & 0.433 & 2.230 & 0.683 \\
  & $1.50\times$ & 0.295 & 2.003 & 0.634 \\
  & $2.00\times$ & 0.382 & 1.880 & 0.707 \\
\midrule
\multirow{3}{*}{Shift}
  & $5\%$  & 0.304 & 1.602 & 0.634 \\
  & $10\%$ & 0.428 & 1.457 & 0.707 \\
  & $20\%$ & 0.532 & 1.899 & 0.732 \\
\midrule
\multicolumn{2}{l}{\textbf{original proposal} (reference)} & 0.386 & 1.852 & 0.683 \\
\multicolumn{2}{l}{\texttt{random} region (floor)}         & 0.147 & 0.277 & 0.537 \\
\bottomrule
\end{tabular*}
\end{table}

\begin{figure}[t]
    \centering
    \includegraphics[width=\linewidth]{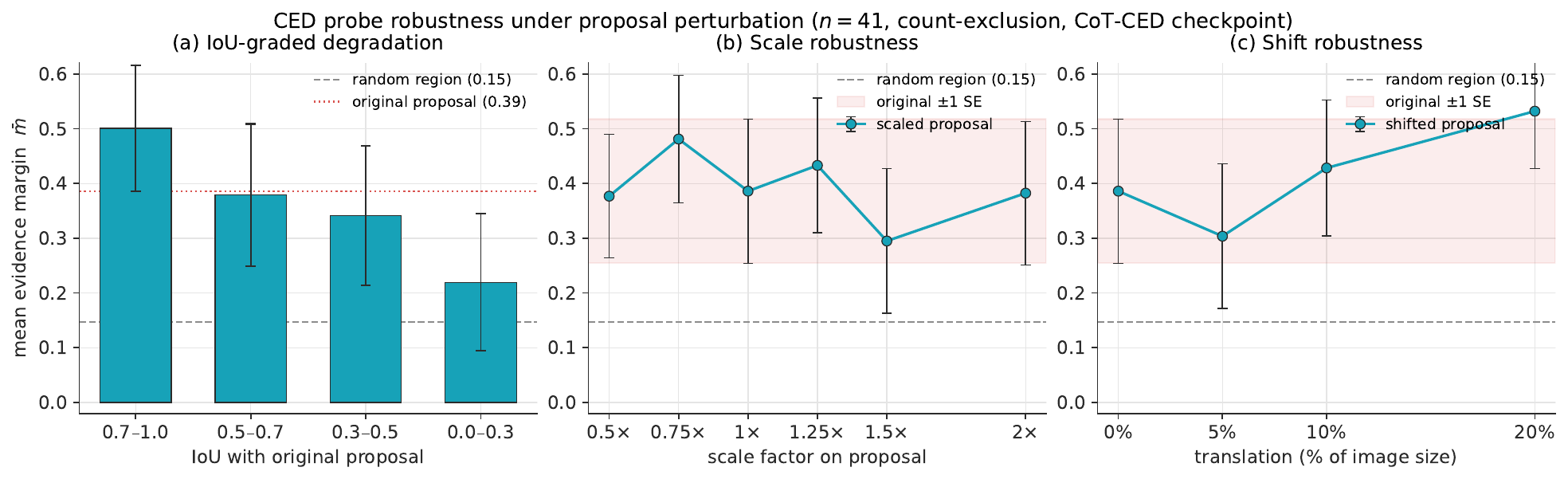}
    \caption{\ced{} probe robustness under proposal perturbation on the count-exclusion diagnostic. (a) Mean evidence margin $\bar m$ across IoU bins, with the random-region baseline as a grey dashed reference. (b,c) $\bar m$ under scale and translation perturbations, with the original $\pm1$ SE band in red. Error bars are $\pm1$ SE.}
    \label{fig:proposal-robustness}
\end{figure}

\paragraph{Interpreting the results.}
IoU gives the cleanest trend: $\bar m$ and $\Delta$ both decrease as overlap with the original proposal drops.
The lowest IoU bin approaches the random floor, indicating that \ced{} does not merely reward masking magnitude.
Scale and shift perturbations do not show systematic degradation within the tested ranges.
Their evidence-reference gaps stay well above the random baseline, suggesting that moderate detector imprecision is a second-order effect for this diagnostic.
All measurements here are probe diagnostics rather than downstream accuracy estimates.

\subsection{Shared-Candidate Baselines and Reward Diagnostics}
\label{app:baseline-comparison}

\paragraph{Unified-condition protocol.}
Recent perception-aware methods differ in backbone and data.
We therefore separate checkpoint transfer from signal design by evaluating lightweight method variants under a shared-candidate protocol with the same backbone, candidate pool, and benchmark manifest.

\begin{table}[t]
\centering
\small
\setlength{\tabcolsep}{6pt}
\caption{Unified-condition baseline comparison on the offline rerank dataset ($N{=}1{,}353$).}
\label{tab:baseline-comparison}
\begin{tabular*}{\linewidth}{@{\extracolsep{\fill}}lccc@{}}
\toprule
\textbf{Method} & \textbf{Reward mean} & \textbf{Signal strength} & \textbf{Signal-corr} \\
\midrule
Vanilla GRPO & $-0.149$ & N/A & N/A \\
PAPO-Lite & $-0.156$ & KL $=0.024$ & $0.034$ \\
\ced{} (Ours) & $+0.448$ & Margin $=-0.007$ & Selective \\
\bottomrule
\end{tabular*}
\end{table}

\paragraph{Shared-candidate comparison.}
\Cref{tab:extended-baseline-comparison} reports Lite versions of representative methods on the counting shared-candidate slice.
The key readout is not final accuracy alone, which is saturated on this slice, but whether the reward can produce non-constant within-group signal and flip candidate preference beyond log-probability ranking.

\begin{table}[t]
\centering
\small
\setlength{\tabcolsep}{4pt}
\caption{Unified-condition method comparison on the counting shared-candidate slice: 94 samples with at least two distinct candidates.}
\label{tab:extended-baseline-comparison}
\begin{tabular*}{\linewidth}{@{\extracolsep{\fill}}lcccc@{}}
\toprule
\textbf{Method} & \textbf{Acc.} & \textbf{Strong inf.} & $r_{\mathrm{nc}}\uparrow$ & flip $\uparrow$ \\
\midrule
LogProb & 94.7 & 75.0 & -- & -- \\
Correctness-only & 97.9 & 100.0 & 12.8 & 4.3 \\
\textbf{\ced{} (Ours)} & \textbf{97.9} & \textbf{100.0} & \textbf{100.0} & \textbf{76.6} \\
Perception-R1-Lite & 97.9 & 100.0 & 13.8 & 5.3 \\
VLM-R1-Lite$^*$ & 94.7 & 75.0 & 0.0 & 0.0 \\
VPPO-Lite$^*$ & 94.7 & 75.0 & 0.0 & 0.0 \\
PAPO-Lite$^*$ & 94.7 & 75.0 & 0.0 & 0.0 \\
\midrule
\multicolumn{5}{l}{\footnotesize $^*$ No reward-based reranking recovered; falls back to log probability.} \\
\bottomrule
\end{tabular*}
\end{table}

\paragraph{Within-group variance telemetry.}
Correctness-only has within-group zero-variance rate 0.949, non-constant rate 0.051, and same-answer-different-reward discrimination 0.0.
Additive and ours reduce the zero-variance rate to 0.035 and restore the non-constant rate to 0.965.
Ours gives slightly higher discrimination, 0.555 versus 0.535.
Re-scoring the pre-RL checkpoint gives a yes/no shortcut bias gap of 0.077 for Correctness-only, 0.075 for Additive, and 0.015 for ours.
The intervention area-reward Pearson correlation is 0.080, ruling out area-driven reward hacking.
Reward-length correlations are consistent across methods: $-0.226$, $-0.222$, and $-0.230$.

\subsection{Reward Discrimination and Failure Modes}
\label{app:same-answer-different-reward}
\label{app:zero-variance-figure}
\label{app:existence-failure}
\label{app:failure-cases}

The same-answer-different-reward statistic in \Cref{tab:signal-validation} is visualized in \Cref{fig:same-answer-diff-reward}.
Rollouts with the same final answer receive different rewards once the evidence margin enters the signal, reflecting differences in the visual evidence each rollout uses.

\begin{figure}[t]
    \centering
    \includegraphics[width=\linewidth]{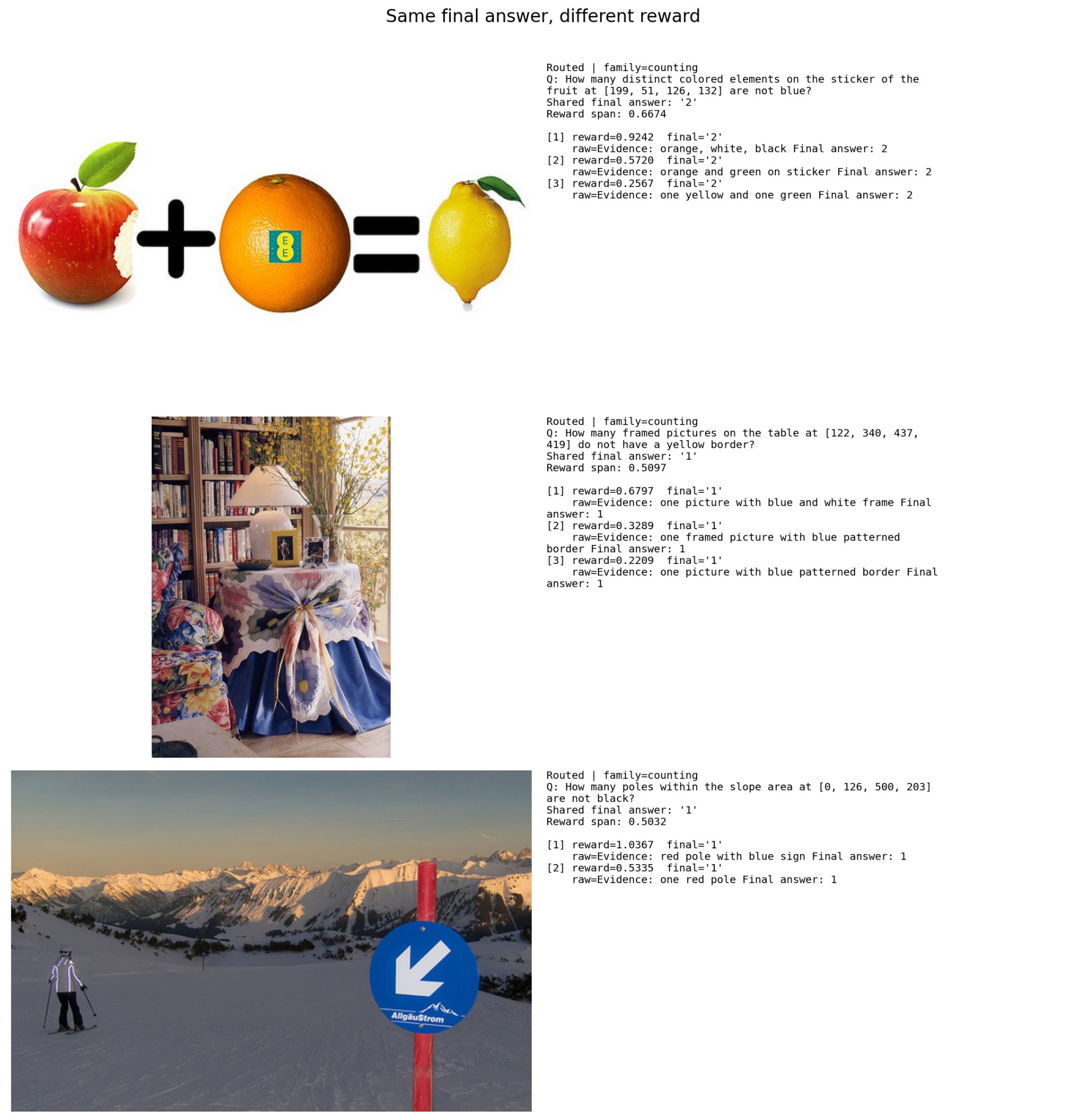}
    \caption{Qualitative examples of same-answer-different-reward under the CED method. All rollouts share the same answer, while \ced{} produces reward spans of 0.50--0.67 according to visual evidence quality.}
    \label{fig:same-answer-diff-reward}
\end{figure}

\begin{figure}[t]
    \centering
    \includegraphics[width=\linewidth]{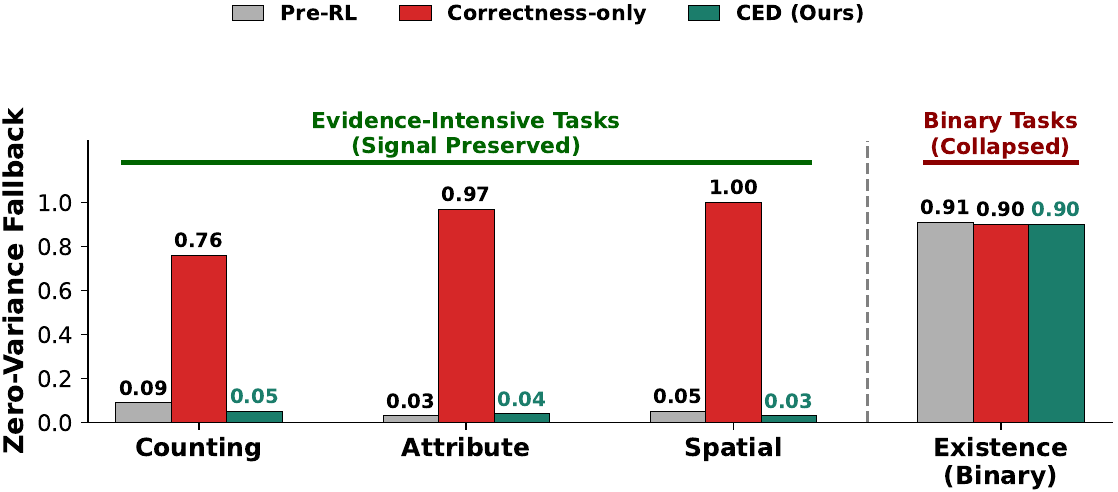}
    \caption{Within-group zero-variance rate across task families. Evidence-intensive tasks show a large reduction from correctness-only reward to our reward; binary-dominated groups remain near high zero-variance because of their two-action structure.}
    \label{fig:zero-variance-appendix}
\end{figure}

\begin{figure}[t]
    \centering
    \includegraphics[width=\linewidth]{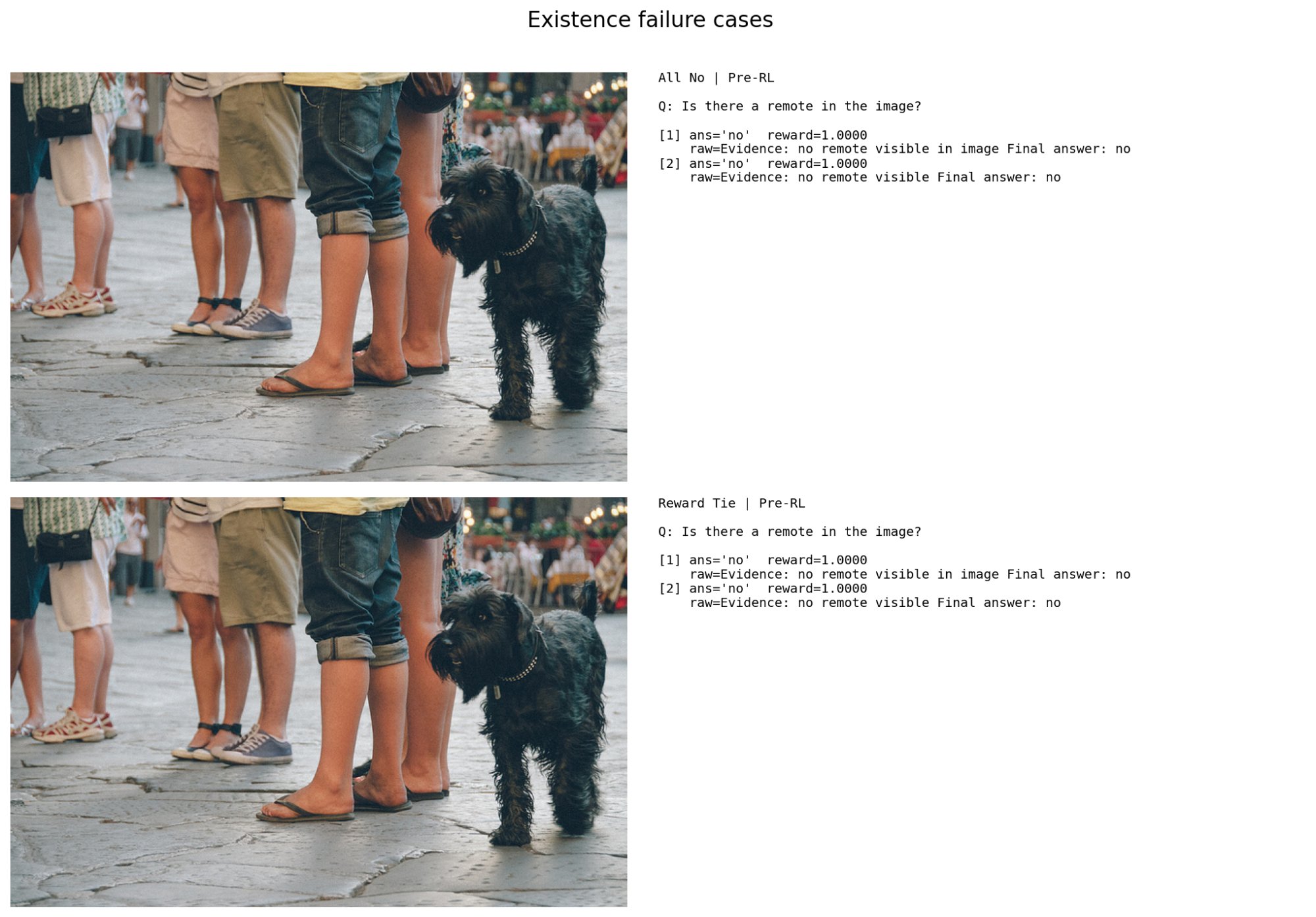}
    \caption{Failure cases on presence tasks: identical negation answers with identical rewards yield zero within-group variance.}
    \label{fig:existence-failure}
\end{figure}

Under Correctness-only, evidence-intensive tasks exhibit low within-group discrimination, with non-constant raw reward rate $r_{\mathrm{nc}}\approx0.05$ and zero-variance rate $\approx0.95$.
Additive and Routed rewards restore discrimination, with $r_{\mathrm{nc}}\approx0.97$ and zero-variance rate $\approx0.04$.
On the presence probe, all formulations remain limited by the binary-task variance bound, supporting the routing of presence tasks to the audit pathway.

\subsection{Scope and Extensions of the Evidence Interface}
\label{app:evidence-interface}

\paragraph{Behavioral target.}
\ced{} is designed to reward a model behavior: a candidate answer should be supported by the visual evidence that justifies it. In the causal view of grounding, this means that the model's support for an answer should decrease more under intervention on the target-evidence path than under matched interventions on non-evidence paths. This requirement is not specific to bounding boxes or object regions.

\paragraph{Region-level instantiation.}
In this paper, we instantiate the audit with object-centric region interventions because regions provide a scalable interface for constructing Evidence Regions and matched non-evidence Regions without question-specific evidence annotations. This choice aligns with the benchmark families studied in our experiments, including counting, spatial reasoning, hallucination detection, and object-centric visual reasoning, where many failures arise when the model answers from language priors, scene context, or salient distractors rather than the relevant local evidence.

\paragraph{Beyond object boxes.} In the causal graph of Figure~1, CED rewards a path-level property: the answer should depend more on the target-evidence path than on shortcut or nuisance paths. Object-region interventions query this property by probing one spatial locus of the evidence path without question-specific evidence annotations, making the audit scalable to open-ended VQA training. The behavioral target itself is defined on the graph structure, not on the intervention format. Richer evidence interfaces such as attribute-preserving interventions, multi-region proposals, or text-span units query the same path at finer granularity or for different evidence types; they refine how the evidence path is probed, not what behavior is rewarded.

\subsection{Answer-CED vs. CoT-CED Details}
\label{app:cot-vs-answer}

The two variants differ in where the evidence sensitivity is computed: CoT-CED applies the counterfactual to the full reasoning trace, while Answer-CED applies it only to the final answer.
Both variants are trained for $2{,}000$ steps on Qwen3.5-9B with LoRA ($r{=}16$, $\alpha{=}32$).
The main paper reports the training trajectories in \Cref{fig:training-dynamics}; here, \Cref{tab:cot-vs-answer} summarizes the final training state and \Cref{tab:cot-behavior,tab:cot-examples} provide the behavior-level diagnostics.

\begin{table}[t]
\centering
\small
\caption{CoT-CED vs. Answer-CED: training-end summary over $2{,}000$ steps.}
\label{tab:cot-vs-answer}
\begin{tabular*}{\linewidth}{@{\extracolsep{\fill}}lcc@{}}
\toprule
\textbf{Metric} & \textbf{CoT-CED} & \textbf{Answer-CED} \\
\midrule
Mean training reward       & $+0.259$ & $-0.009$ \\
Final training correctness & $43.4\%$ & $34.2\%$ \\
\bottomrule
\end{tabular*}
\end{table}

CoT-CED achieves higher training reward (Table~A.12, Figure~\ref{fig:training-dynamics}), yet transfers worse on average (Table~\ref{tab:answer-cot-compare}). This reflects a structural mismatch between dense per-token rewards and GRPO's group-relative optimization. CoT-CED averages evidence rewards over the full reasoning trace, and because trace length is policy-controlled, the model can raise the per-token average by truncating low-reward tokens rather than by improving visual grounding. Tables~A.13--A.14 confirm this collapse: CoT-CED's mean chain shrinks to 3.6 tokens on counting (vs.\ 49.3 for Answer-CED), degenerating to object-cue shorthand. This parallels the process-reward hacking reported by~\citet{guo2025deepseek}, who abandon dense process rewards in favor of outcome-level signals for the same reason. Answer-CED removes this degree of freedom by scoring only the final-answer span, a fixed-length interface whose reward cannot be manipulated through generation length.

\begin{table}[t]
\centering
\small
\caption{End-of-training CoT behavior on held-out prompts. CoT-CED collapses toward an object-cue answer format, while Answer-CED preserves natural multi-sentence reasoning.}
\label{tab:cot-behavior}
\begin{tabular*}{\linewidth}{@{\extracolsep{\fill}}lccc@{}}
\toprule
\textbf{Metric} & \textbf{Benchmark} & \textbf{CoT-CED} & \textbf{Answer-CED} \\
\midrule
Mean CoT length (tokens)        & Counting       & $3.6$  & $49.3$ \\
Median CoT length (tokens)      & Counting       & $3$    & $47$ \\
$p_{90}$ CoT length (tokens)    & Counting       & $6$    & $78$ \\
Mean CoT length (tokens)        & HallusionBench & $3.8$  & $82.1$ \\
Median CoT length (tokens)      & HallusionBench & $3$    & $61$ \\
$p_{90}$ CoT length (tokens)    & HallusionBench & $6$    & $173$ \\
\midrule
Image-referential opener rate   & Counting       & $1.0\%$  & $79.8\%$ \\
Image-referential opener rate   & HallusionBench & $5.2\%$  & $64.2\%$ \\
Unique-token ratio              & Counting       & $99.1\%$ & $75.8\%$ \\
Unique-token ratio              & HallusionBench & $96.8\%$ & $68.1\%$ \\
\bottomrule
\end{tabular*}
\end{table}

\begin{table}[t]
\centering
\footnotesize
\setlength{\tabcolsep}{3pt}
\renewcommand{\arraystretch}{1.08}
\caption{Representative side-by-side generations. Both variants produce the correct final answer, but CoT-CED reduces the chain to an object cue, while Answer-CED keeps an explicit visual check.}
\label{tab:cot-examples}
\begin{tabular*}{\linewidth}{@{\extracolsep{\fill}}>{\raggedright\arraybackslash}p{0.18\linewidth}>{\raggedright\arraybackslash}p{0.24\linewidth}>{\raggedright\arraybackslash}p{0.50\linewidth}@{}}
\toprule
\textbf{Prompt} & \textbf{CoT-CED} & \textbf{Answer-CED} \\
\midrule
How many beds are there in the image? &
\texttt{The beds\,/\,2} &
\emph{Got it, let's look at the image. There's a loft bed with a lower bed. The upper bunk is one, the lower bed is another. So that's two beds.} \\
\addlinespace[1pt]
How many weights are there in the image? &
\texttt{The weights\,/\,Two} &
\emph{Got it, let's look at the image. The man is holding two dumbbells, one in each hand. There's one in the left hand and one in the right hand. So that's two weights.} \\
\bottomrule
\end{tabular*}
\end{table}

\subsection{Formalizing Evidence-Closed Self-Evolution}
\label{app:information-theory}

This appendix provides the complete formalization of the evidence-closed self-evolution regime introduced in \Cref{sec:self-evolve}.

\paragraph{Notation.}
Let $O_t=(I_t,Q_t)$ denote the image-question pair, $\hat Y_t$ the prediction, $V_t=\nu(I_t)$ the task-relevant visual evidence, $G_t\in\{0,1\}$ the grounding indicator, and $S_t$ the endogenous reward.

\begin{definition}[Evidence-closed self-evolution]
\label{def:evidence-closed-app}
A training step is \textbf{evidence-closed} if its signal $T_t^{\mathrm{cl}}$ is measurable with respect to
$\mathcal{B}_t=\sigma(\mathcal{F}_t,I_t,Q_t,\hat Y_t,S_t)$.
\end{definition}

\begin{lemma}[Endogenous reward bottleneck]
\label{lem:bottleneck-app}
If $S_t \perp\!\!\!\perp I_t \mid (\hat Y_t,Q_t,\mathcal{F}_t)$, then
$I(S_t;V_t \mid \hat Y_t,Q_t,\mathcal{F}_t)=0$.
\end{lemma}

\begin{proof}
Since $V_t=\nu(I_t)$, the conditional independence implies
$S_t \perp\!\!\!\perp V_t \mid (\hat Y_t,Q_t,\mathcal{F}_t)$, which gives the result.
\end{proof}

\begin{proposition}[No incremental grounding information]
\label{prop:no-increment-app}
If $T_t^{\mathrm{cl}}$ is evidence-closed, then
$I(T_t^{\mathrm{cl}};G_t \mid \mathcal{B}_t)=0$.
\end{proposition}

\begin{proof}
$T_t^{\mathrm{cl}}$ is $\mathcal{B}_t$-measurable, so conditioning on $\mathcal{B}_t$ leaves no remaining uncertainty in $T_t^{\mathrm{cl}}$ that can carry additional information about $G_t$.
\end{proof}

\begin{proposition}[Closed-loop reward does not identify grounding]
\label{prop:non-identifiable-app}
If two predictors induce the same distribution over $(I_t,Q_t,\hat Y_t,S_t)$ but differ in grounding, then any objective of the form
$J(h)=\mathbb{E}_h[\phi(I_t,Q_t,\hat Y_t,S_t)]$ cannot distinguish them.
\end{proposition}

\begin{proof}
The objective depends only on the joint distribution of $(I_t,Q_t,\hat Y_t,S_t)$.
If two predictors have the same joint distribution over these variables, they have the same value of $J(h)$ regardless of their grounding behavior.
\end{proof}

\begin{assumption}[Audit informativeness]
\label{asmp:audit-info-app}
On evidence-intensive tasks,
$I(A_t;G_t \mid \mathcal{B}_t)>0$,
where $A_t$ denotes the evidence audit signal.
\end{assumption}

\begin{proposition}[Information increment from evidence auditing]
\label{prop:info-increment-app}
Under \Cref{asmp:audit-info-app},
$H(G_t \mid \mathcal{B}_t,A_t) < H(G_t \mid \mathcal{B}_t)$.
\end{proposition}

\begin{proof}
By the definition of conditional mutual information,
$I(A_t;G_t \mid \mathcal{B}_t)=H(G_t \mid \mathcal{B}_t)-H(G_t \mid \mathcal{B}_t,A_t)$.
The assumption makes the left-hand side strictly positive, giving the stated strict inequality.
\end{proof}

\end{document}